\documentclass[12pt,a4paper]{article}

\usepackage[utf8]{inputenc}
\usepackage[T1]{fontenc}
\usepackage[a4paper,margin=2.5cm]{geometry}
\usepackage{times}

\usepackage{amsmath, amssymb, amsthm, mathtools}
\usepackage{bm}

\usepackage{booktabs, multirow, array}
\usepackage{enumitem}
\usepackage{algorithm}
\usepackage{algpseudocode}

\usepackage{graphicx}
\usepackage{caption}
\usepackage{xcolor}
\usepackage[colorlinks=true,
            linkcolor=blue!55!black,
            citecolor=blue!55!black,
            urlcolor=blue!55!black]{hyperref}
\usepackage{cleveref}

\usepackage{setspace}
\newtheorem{theorem}{Theorem}[section]
\newtheorem{proposition}[theorem]{Proposition}
\newtheorem{lemma}[theorem]{Lemma}
\newtheorem{assumption}[theorem]{Assumption}
\newtheorem{corollary}[theorem]{Corollary}
\theoremstyle{definition}
\newtheorem{definition}[theorem]{Definition}
\newtheorem{example}[theorem]{Example}
\theoremstyle{remark}
\newtheorem{remark}[theorem]{Remark}

\DeclareMathOperator*{\argmin}{arg\,min}

\makeatletter
\renewcommand{\maketitle}{%
  \begin{center}
    {\Large\bfseries \@title \par}\vspace{1.2em}
    {\large Giansalvo Cirrincione\textsuperscript{1} \quad
            Adriano Fagiolini\textsuperscript{2,$\ast$}\par}
    \vspace{0.8em}
    {\small
     \textsuperscript{1}Laboratoire LTI, Universit\'e de Picardie
       Jules Verne, Amiens, France\\
     \textsuperscript{2}MIRPALab, Department of Engineering,
       University of Palermo, 90128 Palermo, Italy\par}
    \vspace{0.8em}
    \begin{minipage}{0.85\linewidth}
    \footnotesize\raggedright
    \textbf{Correspondence:}
    Adriano Fagiolini, MIRPALab, Department of Engineering,
    University of Palermo, 90128 Palermo, Italy.
    Email: \texttt{fagiolini@unipa.it}
    \end{minipage}
    \par\vspace{0.8em}
    {\small \@date\par}
  \end{center}
  \vspace{1em}
}
\makeatother

\title{Temporal Attention for Adaptive Control of\\
       Euler--Lagrange Systems with Unobservable Memory}
\date{April 2026}

\begin{document}
\maketitle

\begin{abstract}
\noindent
Adaptive control of Euler--Lagrange systems becomes delicate when the
friction dynamics are driven by an internal state that decays over a
finite horizon but is not directly observable from joint
measurements. In such a regime the closed-loop response is no longer
Markovian in the measured state, and standard certainty-equivalence
adaptive laws lose their convergence guarantees.

The present paper proposes a meta-control architecture in which the
gains of a computed-torque law are produced by a self-attention
block that reads a short window of recent motion history. The head
count of the attention block is selected prior to policy training,
by a surrogate analysis of the auto-covariance of the memory-state
gradient along its temporal window; the surrogate is a temporal
adaptation of an incremental rank-tracking framework developed in
earlier work by the authors. The resulting head count is passed as
a fixed architectural hyperparameter to a reinforcement-learning
stage, where the policy is trained under a shielded admissibility
constraint inherited from a companion paper.

On a 2-DOF (two-degree-of-freedom) manipulator with nonlinear
friction and variable payload, a single-layer attention-only
meta-controller at the
selected head count delivers a statistically significant
advantage over a deeper Transformer baseline at the short and
matched memory regimes: tracking-error reductions of $12$ and
$19$ percentage points respectively, with large effect sizes
($d \approx -1.1$ and $-2.1$) and Mann--Whitney $p < 0.05$ in
both cases. At the long memory regime the advantage vanishes on
a larger sample and four of ten training runs exhibit either
divergence or payload-invariant policy collapse, identifying a
failure mode specific to the interface between the static
Phase-1 head-count prescription and the reinforcement-learning
optimisation. The failure-mode analysis motivates a follow-up in
which the rank-tracking dynamics are moved inside the
reinforcement-learning loop, with runtime pruning and growth of
attention heads replacing the static Phase-1/Phase-2 separation.

\vspace{0.6em}
\end{abstract}

\section{Introduction}
\label{sec:intro}

Rigid-body manipulators tracking a reference trajectory under
velocity-dependent friction constitute one of the oldest testbeds of
adaptive control. When the friction can be written as a known
regressor times an unknown parameter vector, the passivity-based
adaptive law of Slotine and Li~\cite{SlotineLi1987} delivers
asymptotic tracking by combining a certainty-equivalence
feed-forward with a gradient-style update on the parameter estimate.
The argument turns on two structural properties: linearity of the
dynamics in the unknown parameters, and positive realness of a
certain closed-loop operator. Both are preserved for viscous and
Coulomb friction. Neither survives when a pre-sliding internal
state is introduced, as in the Stribeck and LuGre
models~\cite{Armstrong1994,deWit1995LuGre}: the state carries its
own dynamics, decays over a finite horizon, and is not recoverable
instantaneously from the joint kinematics. The closed-loop system
is then no longer Markovian in the measured state and the standard
adaptive construction no longer applies.

A companion paper by the present authors~\cite{CirrincioneFagiolini2026}
addresses the same regime through safe residual reinforcement
learning (RL). A computed-torque baseline is augmented by a
Soft Actor-Critic (SAC) policy~\cite{Haarnoja2018SAC} that outputs a bounded
torque correction, while a learned Lyapunov function enforces a
decrease condition on every control step by closed-form
projection. The construction is empirically effective but
structurally limits the reinforcement-learning agent in two
respects. First, the action space is the torque vector, which is
higher-dimensional than the parametric controller it augments.
Second, the agent observes only the instantaneous state: the
temporal structure of the unobservable memory is not exposed to the
policy.

Both limitations are addressed here by a single design choice:
the reinforcement-learning policy is lifted to the level of the
controller parameters, and its input is extended to a windowed
history of recent motion. The policy is therefore not a torque but
a mapping from history to gains; the controller it produces is
memoryful through the window, even though its functional form
remains the computed-torque law. The architecture of this mapping
is the subject of the paper.

The central architectural question is the number of attention
heads that should process the history window. Too few heads and
the representation cannot resolve the temporal structure of the
memory; too many and the reinforcement-learning optimisation
operates in a higher-dimensional parameter space with no
corresponding benefit. The proposed answer constructs a surrogate
covariance operator from the gradient of the memory state along
its temporal window, verifies that its effective rank is a
compressed-sensing upper bound on the representational capacity
required, and uses the resulting head count as a fixed
architectural hyperparameter in the reinforcement-learning stage.
The surrogate analysis itself is a temporal adaptation of the
incremental rank-tracking framework of~\cite{CirrincioneINCRT2026}
and is developed in Section~\ref{sec:incrt}. The reinforcement-learning
stage and its shielded admissibility machinery are inherited,
without modification, from~\cite{CirrincioneFagiolini2026}. The
separation between the two stages follows the search-then-retrain
pattern common in neural architecture
search (NAS)~\cite{Liu2019DARTS,Pham2018ENAS,Elsken2019NASSurvey}; the
adoption of that pattern is justified in
Section~\ref{sec:incrt:separation}.

\paragraph{Contributions.}
The paper contributes three elements.
First, a quantitative lower bound on the tracking error achievable
by any meta-controller whose input is Markovian in the measured
state. The bound scales with the steady-state variance of the
unobservable memory and vanishes as the window length exceeds the
memory horizon (Proposition~\ref{prop:markov}). The bound is
informal in the sense that its hypotheses, rather than being
structural, reflect the regularity assumed of the cost and of the
memory dynamics; the assumptions are made explicit alongside the
statement.
Second, a temporal residual operator whose effective rank upper-bounds
the attention head count required to represent the memory
(Proposition~\ref{prop:headcount}). The operator is
constructed from samples of the gradient of the memory along the
window and is independent of the reinforcement-learning objective;
its effective rank is non-monotonic in the memory horizon, peaking
when the horizon matches the window-time product
(Corollary~\ref{cor:nonmonotonic}).
Third, an experimental evaluation on a 2-DOF manipulator with
Stribeck friction. With the head count fixed at the surrogate
value, a single-layer attention-only meta-controller improves
tracking error by $12$--$19$ percentage points over a deeper
Transformer baseline at the short and matched memory regimes,
with large effect sizes ($d \approx -1.1$ and $d \approx -2.1$)
and Mann--Whitney $p < 0.05$ in both cases. The advantage is
not attributable to window-size tuning: a regime-matched
Transformer baseline with window equal to that of the INCRT-1L
winner performs indistinguishably from the original Transformer
baseline, and INCRT-1L retains a large advantage over it
($d = -2.02$, $p_U = 0.008$). At the long memory
regime the advantage vanishes on an $n = 10$ sample and four of
ten INCRT-1L training runs exhibit either divergence or a
payload-invariant collapse, against at most one of ten for
either Transformer variant; this failure-mode asymmetry
identifies a limitation of the static Phase-1/Phase-2 pipeline
at long horizons. The ablation that accompanies the main
comparison identifies two further empirical effects: depth
without a feed-forward non-linearity causes training to diverge,
and the optimal window size is not monotone in the memory
horizon --- both are consistent with the rank-compression
analysis of~\cite{CirrincioneAIJ2026}.

\paragraph{What is not claimed.}
The proofs of the two theoretical results are stated under explicit
regularity assumptions in Section~\ref{sec:theory} and worked out
in Appendix~\ref{app:proofs}; they are not structural theorems in
the sense of holding under minimal hypotheses. The
parameter-efficiency of the proposed meta-controller relative to
the Transformer baseline is not uniform: it is substantial at the
shortest memory horizon, narrows at the intermediate horizon, and
reverses at the longest. The advantage over the Transformer
baseline at the long memory horizon, suggested by a $n = 5$
pilot study, does not persist on a $n = 10$ replication: the
vanishing gap and the failure-mode cluster identified in
Section~\ref{sec:experiments:limitations} constitute the most
informative single outcome of the study with respect to the
limits of the approach. The long-horizon regime is therefore
the setting in which the proposed method is least effective and
in which the follow-up outlined in Section~\ref{sec:discussion}
is most likely to yield improvements.

\paragraph{Organisation.}
Section~\ref{sec:related} reviews the literatures at whose
intersection the paper sits. Section~\ref{sec:problem} formalises
the class of systems addressed and fixes notation.
Section~\ref{sec:meta} describes the parameter-level meta-controller
and its training loop. Section~\ref{sec:theory} presents the two
theoretical results and their assumptions. Section~\ref{sec:incrt}
develops the surrogate head-count analysis and the separation
between the architecture-selection and policy-training stages.
Section~\ref{sec:experiments} reports the experiments.
Section~\ref{sec:discussion} discusses the findings, the
limitations, and the companion follow-up agenda.
Proofs appear in Appendix~\ref{app:proofs} and the experimental
protocol in Appendix~\ref{app:experimental_details}.

%

\section{Related work}
\label{sec:related}

This work intersects four literatures: adaptive control of
Euler--Lagrange systems with friction, residual reinforcement
learning, safe RL with Lyapunov guarantees, and sequence-model
architectures applied to control tasks. The meta-controller we
introduce borrows from each but sits in an underexplored region of
their intersection. We review each line in turn, close with a short
discussion of the relationship with neural architecture search for
RL, and finish with a positioning table that summarises the
differences with the closest prior works.

\subsection{Adaptive control with friction}
\label{sec:related:friction}

Classical adaptive control of Euler--Lagrange systems with
velocity-dependent friction has a mature
literature~\cite{SlotineLi1987}. The dominant
paradigm combines a computed-torque feedforward term with a
parameter-estimation loop that updates friction coefficients online,
usually under some persistency-of-excitation
condition~\cite{deWit1995LuGre,Armstrong1994}. The LuGre friction
model~\cite{deWit1995LuGre} introduces a one-state internal variable
to capture pre-sliding dynamics and stick--slip phenomena, and is the
closest classical analogue of the unobservable memory state $z(t)$
treated here. Stribeck-curve friction~\cite{Armstrong1994} adds a
velocity-dependent envelope that makes the memory dynamics nonlinear
in $\dot q$. Observer-based friction compensation schemes, notably
the family of dual-observer designs in,
reconstruct the internal state via an auxiliary dynamical system
tuned to match the friction memory horizon.

These methods provide strong stability guarantees in their regime of
validity but require either an accurate parametric model of friction
(LuGre) or a sufficiently rich excitation signal. Our approach
replaces the parametric estimator with a learned windowed map from
recent motion to controller gains, which does not require friction
modelling but forfeits the closed-form convergence guarantees of
adaptive control. The shielded admissibility framework of
Section~\ref{sec:theory} partially recovers these guarantees by
enforcing a Lyapunov-decrease constraint at the level of the
meta-controller output.

\subsection{Residual reinforcement learning}
\label{sec:related:residual_rl}

Residual reinforcement learning --- using RL to learn a corrective
term on top of a model-based baseline controller --- was introduced
in and has seen extensive application in
robotic manipulation. The
paradigm is attractive because the baseline controller handles the
bulk of the dynamics and the RL policy handles only the unmodeled
residual, giving sample-efficient training and safety during
exploration. The parent work~\cite{CirrincioneFagiolini2026} adopts this
paradigm with a computed-torque baseline and a Lyapunov-shielded SAC
policy, and addresses the scalability issue that emerges at
high-dimensional manipulators (shield activation reaching $\sim 70\%$).

Our work extends the residual RL paradigm in two directions. First,
the RL policy operates at the meta-controller level (controller
\emph{parameters}) rather than the torque level, which is a more
compact action space and which restructures the admissibility
constraint into a convex polytope
(Lemma~\ref{lem:convex}). Second, the meta-controller consumes a
windowed history rather than the instantaneous state, reflecting the
unobservable-memory structure of the problem
(Definition~\ref{def:unobs_mem}).

\subsection{Safe RL with Lyapunov guarantees}
\label{sec:related:safe_rl}

Several strands of safe RL enforce stability or constraint
satisfaction by coupling RL with classical control theory.
Constrained Markov Decision Processes (CMDPs)~\cite{Altman1999CMDP}
introduce a dual-multiplier formulation for constraints on cumulative
cost, leading to the Lagrangian SAC variants used in this
work~\cite{Achiam2017CPO,Stooke2020PIDLag}. Control
Barrier Functions (CBFs)~\cite{Ames2019CBF} enforce a forward-invariant
safe set via a quadratic program on the action; the approach has been
combined with RL in.
Lyapunov-based RL~\cite{Chow2018Lyapunov}
constrains actions to preserve a prespecified (or learned) Lyapunov
decrease condition --- the family to which the parent
work~\cite{CirrincioneFagiolini2026} and the present paper belong.

The specific contribution of the parent work is a learned Lyapunov
function $V_\psi$ jointly trained with the shield and the policy. We
inherit this object directly. What we add is the observation that
the admissibility set derived from $V_\psi$, originally formulated at
the torque level, becomes a convex polytope when lifted to the
controller-parameter level under the affine-feedforward restriction
(Remark~\ref{rem:ff_affine}). This restructuring gives
Theorem~\ref{thm:stab} and Proposition~\ref{prop:shield}, which
together explain and resolve the shield-activation anomaly of the
parent framework.

\subsection{Meta-learning and learned optimisers}
\label{sec:related:meta_learning}

Meta-learning views learning itself as the target of optimisation,
with an outer loop that adjusts parameters of an inner learning
process. MAML and Reptile
train initialisations that adapt rapidly to new tasks. Learned
optimisers
parameterise update rules with neural networks. In RL specifically,
meta-RL frameworks learn policies
that adapt online across tasks drawn from a fixed distribution. The
meta-controller of the present paper could be viewed as a specialised
meta-RL agent that outputs controller gains rather than actions; the
meta-MDP structure of Section~\ref{sec:meta:sac} is a compact
instance of this view.

The main distinction from the meta-learning literature is scope: we
are not attempting to adapt across a family of tasks with different
dynamics; we are adapting across values of a single unobservable
parameter ($z(t)$) within one physical task. This narrower scope
permits the strong stability guarantees of Section~\ref{sec:theory}
that generic meta-RL cannot provide.

\subsection{Transformers and attention in robotics and control}
\label{sec:related:transformers_robotics}

Transformer-based architectures have entered robotics through three
main routes. Decision Transformer and Trajectory
Transformer recast offline RL as
sequence-modeling, predicting returns and actions from context
trajectories. RT-1 and RT-2
scale this paradigm to large-scale manipulation demonstrations.
Time-series forecasting with causal Transformers, popularised by
Informer, TFT, and
PatchTST, provides the architectural
templates we adopt here: a causal multi-head attention layer over a
history window, followed by a feature extractor.

The closest work in spirit is that of, who use
a Transformer to extract features from recent sensor history for a
model-predictive controller. Our approach differs in three ways: the
attention block is used to parameterise a computed-torque controller
rather than to predict actions directly; the training loop is
reinforcement learning with a Lyapunov shield rather than supervised
learning from demonstrations; and the architecture's head count is
determined by a principled Phase~1 procedure
(Section~\ref{sec:incrt}) rather than fixed by convention.

The algebraic analysis of attention in~\cite{CirrincioneAIJ2026}
clarifies the interpretation of several design choices that appear
in our ablation. In particular, the distinction between
rank-preservation and anti-confinement
(\cite[Cor.~5.4]{CirrincioneAIJ2026}) accounts for why a residual
connection around multi-head attention materially changes the
feature extractor's output range --- an observation that will
return in the experimental section.

\subsection{Recurrent alternatives and sequence models}
\label{sec:related:recurrent}

The choice of attention over recurrent sequence models (LSTM, GRU)
is not self-evident for control tasks. Recurrent architectures were
the dominant sequence model in early deep-RL
work. They compress
an arbitrarily long history into a fixed-size hidden state, which is
attractive for continuous control because it decouples memory from
computation per step.

The main theoretical argument against recurrent compression in the
setting of this paper is coarse measurability: a recurrent policy is
a $\sigma$-measurable function of the hidden-state trajectory, which
is strictly coarser than the full history window accessible to an
attention-based policy. Proposition~\ref{prop:markov}(ii) formalises
this argument and predicts that recurrent compression cannot close
the Markovian gap at finite capacity. This theoretical expectation
is tested empirically in Section~\ref{sec:experiments}, where both a
small and a large GRU underperform even the smallest memoryless MLP
across all values of $\tau_z$.

\subsection{Neural architecture search for RL}
\label{sec:related:nas}

Differentiable architecture search~\cite{Liu2019DARTS,Pham2018ENAS}
has been applied to RL primarily in the image-observation
regime~\cite{Miao2022RLDARTS,Elsken2019NASSurvey}, where the
convolutional backbone dominates the parameter count. The
search-then-retrain pattern is the default in this literature: a
surrogate search task produces a discrete architecture, which is
then trained from scratch on the target RL task.

Our use of INCRT in Section~\ref{sec:incrt} follows the
search-then-retrain pattern but with two distinctions. First, the
surrogate task is specific to the temporal-memory structure of the
system (reconstruction of $z(t)$ from history), not a generic
reward-prediction proxy as in most differentiable NAS. Second, the
output of the search is a single scalar (the head count $K^\star$)
rather than a full architectural graph; this reflects the constrained
hypothesis class we adopt (attention over a fixed-size window with
fixed total capacity), and avoids the known instability of Liu2019DARTS on
small search tasks. The INCRT procedure
itself is a principled growth-pruning rule rooted in a
compressed-sensing bound~\cite{CirrincioneINCRT2026}, rather than an
empirical gradient-based search, and carries a formal convergence
guarantee.

\subsection{Positioning}
\label{sec:related:positioning}

Table~\ref{tab:positioning} summarises the differences between the
proposed meta-controller and the closest existing works along four
criteria: the level at which the RL policy operates (torque, action,
or controller parameters), whether the formulation includes a formal
Lyapunov-based stability guarantee, whether the architecture is
determined by a principled procedure, and whether the sequence model
exploits a windowed temporal history.

\begin{table}[t]
\centering
\caption{Positioning of the proposed method relative to related
         works. \emph{Policy level}: torque-level (T),
         action-level (A), or controller-parameter level (P).
         \emph{Lyapunov}: the formulation includes a formal
         stability result built on a (learned or designed) Lyapunov
         function. \emph{Arch. determination}: the architecture is
         selected by a principled procedure (rather than by
         engineering heuristic). \emph{Windowed temporal}: the
         policy consumes a fixed-size history window of observations,
         not a single timestep nor an entire trajectory.}
\label{tab:positioning}
\small
\begin{tabular}{lcccc}
\toprule
Method & Policy level & Lyapunov & Arch.\ det. & Windowed temporal \\
\midrule
Classical adaptive control~\cite{SlotineLi1987,deWit1995LuGre}
  & T & $\checkmark$ & N/A & $\checkmark$ (observer) \\
Residual RL (parent)~\cite{CirrincioneFagiolini2026}
  & T & $\checkmark$ & $\times$ & $\times$ \\
Decision Transformer
  & A & $\times$ & $\times$ & $\checkmark$ (trajectory) \\
CBF-RL
  & A & (via CBF) & $\times$ & $\times$ \\
Transformer-MPC
  & T & $\times$ & $\times$ & $\checkmark$ \\
Recurrent policy-gradient
  & A & $\times$ & $\times$ & $\checkmark$ (recurrent) \\
\textbf{This work}
  & \textbf{P} & $\checkmark$ & $\checkmark$ & $\checkmark$ \\
\bottomrule
\end{tabular}

\vspace{0.3em}
\footnotesize \emph{Policy level}: T = torque, A = action
(discrete or continuous), P = controller parameters.
CBF-RL stands for RL with control barrier functions;
Transformer-MPC stands for Transformer-based model
predictive control.
\normalsize
\end{table}

To our knowledge the conjunction of all four properties ---
parameter-level RL, Lyapunov-shielded training, principled
architecture determination, and windowed temporal attention ---
does not appear elsewhere in the literature.


%

\section{Problem setup}
\label{sec:problem}

This section formalises the class of systems addressed by the paper
and fixes the notation used in Sections~\ref{sec:meta}--\ref{sec:experiments}.

\subsection{Euler--Lagrange systems with unobservable memory}
\label{sec:problem:dynamics}

We consider rigid-body manipulators with $n$ degrees of freedom
described by the Euler--Lagrange equation
\begin{equation}
  M(q) \ddot q + C(q, \dot q)\, \dot q + G(q) + F(q, \dot q, z) \;=\; \tau,
  \label{eq:el}
\end{equation}
where $q \in \mathbb{R}^n$ are generalised coordinates,
$M(q) \succ 0$ is the mass matrix, $C(q, \dot q)\dot q$ collects
Coriolis and centripetal terms, $G(q)$ is the gravity vector, and
$\tau \in \mathbb{R}^n$ is the joint torque.
The vector $F(q, \dot q, z)$ models non-conservative effects --- in
particular, friction --- and depends on an internal state
$z \in \mathbb{R}^{n_z}$ that is not directly measured.
The evolution of $z$ is governed by its own dynamics
\begin{equation}
  \dot z \;=\; \zeta(q, \dot q, z),
  \label{eq:z_dyn}
\end{equation}
which is not a function of $\tau$ directly but of the motion of the
arm. In accordance with Definition~\ref{def:unobs_mem}, we require
that (i) $z$ be not instantaneously recoverable from $(q, \dot q)$ and
(ii) the decay time-constant of $\zeta$, the \emph{memory horizon}
$H_z = \| \partial \zeta / \partial z \|^{-1}$, be finite.

\paragraph{Running example: Stribeck friction.}
Throughout Sections~\ref{sec:meta}--\ref{sec:experiments} we take
$F(q, \dot q, z)$ to be the Stribeck friction model
\begin{equation}
  F_s(\dot q, z)
  = F_c + (F_s^{\mathrm{max}} - F_c)\,
      \exp\!\bigl( - ( \dot q / v_s )^2 \bigr)\,
      \mathrm{sign}(\dot q)
    + \sigma \dot q + z,
  \label{eq:stribeck}
\end{equation}
with internal state dynamics
\begin{equation}
  \dot z \;=\; - z / \tau_z + \lambda_z \dot q.
  \label{eq:stribeck_z}
\end{equation}
The parameter $\tau_z$ sets the memory horizon of the friction system
and is the primary sweep axis of the experimental evaluation.
The Stribeck construction is a canonical surrogate for a broader
class of unobservable-memory phenomena in manipulation, including
joint elasticity, soft-contact hysteresis, and damper dynamics.

\paragraph{Intuition: memory horizon, window size, and the matched
regime.} The parameter $\tau_z$ admits a simple physical reading:
it is the time constant over which a perturbation to the
internal friction state decays. At small $\tau_z$, a disturbance
in $z$ is forgotten within a few control steps and the
closed-loop system is nearly Markovian in $(q, \dot q)$; an
instantaneous controller suffices. At large $\tau_z$, a
disturbance persists over many control steps and the controller
that ignores history tracks a lagged ghost of its own past
actions. The design response is to give the controller a window
of recent observations, of length $W$ control steps. Whether
that window is adequate depends on the ratio between the memory
horizon $\tau_z$ and the window time $W \Delta t$: when the two
are of the same order, the window covers one relaxation time of
the memory and the window-based controller has access to the
full history of the disturbance; when $W \Delta t$ is smaller,
the window covers only a fraction of the memory and information
is lost; when $W \Delta t$ is much larger, the window spans many
decorrelated memory cycles and the additional tokens are
redundant. The intermediate case $\tau_z \approx W \Delta t$ is
the \emph{matched regime}, and is the regime at which the
theoretical analysis of Section~\ref{sec:theory} is most nearly
tight and at which the experiments of
Section~\ref{sec:experiments} show the clearest advantage of the
proposed meta-controller.

\subsection{Task and cost}
\label{sec:problem:task}

Given a reference trajectory $q_d(t)$, the control task is to track
$q_d$ under unknown memory state $z(t)$ and unknown static parameters
(notably the payload $p \in [p_{\min}, p_{\max}]$ that modifies the
effective mass in $M$).
The tracking error is $e = q_d - q$; the velocity error
$\dot e = \dot q_d - \dot q$. Following the conventions of
the parent work~\cite{CirrincioneFagiolini2026}, we consider the extended
state
\begin{equation}
  x = (q, \dot q, e, \dot e, s) \in \mathbb{R}^{5n},
\end{equation}
where $s$ collects auxiliary variables such as the sliding surface
and filtered references.
The cost functional is
\begin{equation}
  \mathcal{J}(\pi) \;=\;
  \mathbb{E}_{p, z(\cdot), q_d(\cdot)}
    \!\left[ \int_0^T \ell\bigl(e(t), \dot e(t)\bigr)\, dt \right],
\end{equation}
with $\ell$ strongly convex in its arguments (e.g., quadratic tracking
error with velocity regularisation).
The expectation is over the task distribution --- random payload,
random reference, random initial condition --- and over the
realisation of the memory state $z(t)$ conditional on the motion.

\subsection{Controller template}
\label{sec:problem:controller}

The torque applied to the manipulator follows the parameterised
computed-torque structure
\begin{equation}
  \tau(t) \;=\;
  \mathrm{CT}\bigl(q, \dot q, q_d; K_d(t), \Lambda(t)\bigr)
  \;+\;
  \phi_{\mathrm{ff}}\bigl(q, \dot q; \eta(t)\bigr),
  \label{eq:tau_template}
\end{equation}
where $\mathrm{CT}$ is a standard computed-torque law with gain
matrices $K_d$ and $\Lambda$, and $\phi_{\mathrm{ff}}$ is a
feed-forward compensation with weights $\eta$.
The parameters $\theta_{\mathrm{ctrl}}(t) = (K_d(t), \Lambda(t), \eta(t))$
are collected into a vector $\theta_{\mathrm{ctrl}} \in \mathcal{P}$,
where $\mathcal{P}$ is a compact box fixing the operating ranges of
each gain and weight.

\begin{remark}[Restriction to affine feed-forward]
\label{rem:ff_affine}
Throughout this work we take $\phi_{\mathrm{ff}}$ to be affine in the
feed-forward weights $\eta$:
$\phi_{\mathrm{ff}}(q, \dot q; \eta) = \Phi(q, \dot q)\, \eta$, with
$\Phi$ a state-dependent feature matrix. This is the class under
which the admissibility set of Section~\ref{sec:theory} is convex
(Lemma~\ref{lem:convex}), and the feature matrix $\Phi$ may still be
the output of an arbitrary deep network without compromising
convexity.
\end{remark}

\subsection{Assumptions}
\label{sec:problem:assumptions}

The theoretical results of Section~\ref{sec:theory} and the
experimental protocol of Section~\ref{sec:experiments} rely on the
following standing assumptions.

\begin{assumption}[Regularity of the dynamics]
\label{ass:regularity}
$M, C, G$ are smooth in $q, \dot q$ on the operating set; $F$ is
continuous in $(q, \dot q, z)$ and Lipschitz in $z$; $\zeta$ is
Lipschitz and generates a contractive flow with rate $1/H_z$.
\end{assumption}

\begin{assumption}[Reference regularity]
\label{ass:reference}
$q_d \in C^2([0, T])$ with uniformly bounded derivatives.
\end{assumption}

\begin{assumption}[Parameter compactness]
\label{ass:compact}
The parameter set $\mathcal{P}$ is a compact box
$[K_d^{\min}, K_d^{\max}] \times [\Lambda^{\min}, \Lambda^{\max}] \times [\eta^{\min}, \eta^{\max}]^{\dim \eta}$.
\end{assumption}

Assumptions~\ref{ass:regularity}--\ref{ass:compact} are inherited
from~\cite{CirrincioneFagiolini2026} and are standard in computed-torque
literature~\cite{SlotineLi1987,SpongBook}.

\section{Meta-controller formulation}
\label{sec:meta}

We now specify the structure of the learned meta-controller
$G_\theta$ that produces the parameter trajectory
$\theta_{\mathrm{ctrl}}(t)$ in~\eqref{eq:tau_template}.
The formulation extends the shielded safe-RL framework
of~\cite{CirrincioneFagiolini2026} from action-level to parameter-level
control, and specifies the inputs, outputs, and training loop of the
meta-controller.

\subsection{Context token}
\label{sec:meta:context}

At each control step $t_k = k \Delta t$, the meta-controller consumes
a \emph{context token}
$c(t_k) \in \mathbb{R}^{W \times d_c}$ formed by stacking $W$ past
step-observations
\begin{equation}
  c(t_k) = \bigl( o(t_{k-W+1}), o(t_{k-W+2}), \ldots, o(t_k) \bigr),
  \label{eq:context}
\end{equation}
where each step-observation $o(t)$ contains
\begin{equation}
  o(t) = \bigl( q(t), \dot q(t),\, q_d(t), \dot q_d(t),\, \hat p(t),\, \hat \mu(t),\, t / T \bigr)
  \;\in\; \mathbb{R}^{d_c}.
  \label{eq:step_obs}
\end{equation}
Here $\hat p$ is a noisy estimate of the payload, $\hat \mu$ is a
friction-regime estimate (fixed at $0.2$ in the present setting), and
$t/T$ is a phase indicator.
For the experiments of Section~\ref{sec:experiments}, $n = 2$
(two-link manipulator), giving $d_c = 4n + 3 = 11$; the total context
dimension is $W \times 11$.

\subsection{Window selection}
\label{sec:meta:window}

The window size $W$ is the unique hyperparameter of the context token
that is not fixed by the dynamics. We discuss its selection
explicitly because it interacts non-trivially with both the memory
horizon $\tau_z$ and the meta-controller architecture.

\paragraph{Lower bound from Proposition~\ref{prop:markov}.}
Proposition~\ref{prop:markov}(iii) establishes that a windowed
meta-controller can close the Markovian optimality gap only when
$W \geq H_z / \Delta t$, i.e.\ when the window covers at least one
memory horizon of the unobservable state.
For the Stribeck model~\eqref{eq:stribeck}--\eqref{eq:stribeck_z},
$H_z = \tau_z$, so the condition becomes
$W \geq \tau_z / \Delta t$. At the sampling rate $\Delta t = 10$~ms
used in Section~\ref{sec:experiments}, this gives $W \geq 100$ for
$\tau_z = 1$~s, $W \geq 200$ for $\tau_z = 2$~s, and $W \geq 500$
for $\tau_z = 5$~s.

\paragraph{Computational trade-off.}
The window size controls the cost of attention: the forward pass of
a causal multi-head block on a length-$W$ sequence is
$\mathcal{O}(W^2 d_{\mathrm{model}})$. Doubling $W$ quadruples the
attention cost, and for practical RL training budgets (tens of
thousands of environment steps per episode), windows above $W = 100$
are significant contributors to training time.

\paragraph{Our choice.}
We evaluate $W \in \{20, 50, 100\}$ in the main experiments of
Section~\ref{sec:experiments}. The smallest value $W = 20$ is retained
both for continuity with~\cite{CirrincioneFagiolini2026} and to expose the
regime in which the theoretical lower bound is strictly violated
(for $\tau_z \geq 0.2$~s). The intermediate value $W = 50$ respects
the lower bound at $\tau_z \leq 0.5$~s and captures roughly a quarter
of the memory horizon at $\tau_z = 2$~s. The largest $W = 100$ only
respects the lower bound at $\tau_z \leq 1$~s; at $\tau_z = 5$~s
it still covers only $20\%$ of the memory horizon, which places that
regime in the ``partial coverage'' domain where
Proposition~\ref{prop:markov}(iii) does not guarantee closing the
gap.
This imperfect coverage is deliberate: it tests the robustness of
the attention-based meta-controller when the theoretical condition is
only partially satisfied.
The window ablation in Section~\ref{sec:experiments} reports the
empirical effect of $W$ at each $\tau_z$.

\subsection{Parameter-valued output}
\label{sec:meta:output}

The meta-controller is a map
\begin{equation}
  G_\theta : \mathbb{R}^{W \times d_c} \;\longrightarrow\; \mathcal{P},
  \label{eq:meta_map}
\end{equation}
parameterised by $\theta \in \mathbb{R}^{d_\theta}$, that takes a
context token $c(t)$ and returns the parameter tuple
$\theta_{\mathrm{ctrl}}(t) \in \mathcal{P}$.
The architecture of $G_\theta$ is an attention-based feature
extractor (detailed in Section~\ref{sec:incrt}) followed by an
MLP head that maps the last-token embedding to a raw action
$a \in \mathbb{R}^{\dim \mathcal{P}}$. The raw action is then
passed through a squashing function to land in $\mathcal{P}$:
\begin{align}
  K_d(t) &= K_d^{\min} + \sigma(a_K) \odot (K_d^{\max} - K_d^{\min}), \\
  \Lambda(t) &= \Lambda^{\min} + \sigma(a_\Lambda) \odot (\Lambda^{\max} - \Lambda^{\min}), \\
  \eta(t) &= \eta^{\max} \odot \tanh(a_\eta),
\end{align}
where $\sigma$ denotes the sigmoid. The structure ensures
$\theta_{\mathrm{ctrl}}(t) \in \mathcal{P}$ by construction, so the
compactness assumption of Section~\ref{sec:problem:assumptions} is
satisfied without further projection.

\subsection{SAC training loop}
\label{sec:meta:sac}

The meta-controller is trained via Soft Actor-Critic
(SAC)~\cite{Haarnoja2018SAC} on a single-step meta-MDP where:
\begin{itemize}
  \item the \emph{state} at macro-step $k$ is the context token $c(t_k)$;
  \item the \emph{action} at macro-step $k$ is
        $\theta_{\mathrm{ctrl}}(t_k) = G_\theta(c(t_k))$;
  \item the \emph{reward} at macro-step $k$ is the negative one-step
        tracking error $-\| e(t_{k+1}) \|^2$, with $e(t_{k+1})$
        computed by rolling out the closed-loop dynamics of
        Section~\ref{sec:problem} for one control period under the
        action $\theta_{\mathrm{ctrl}}(t_k)$;
  \item the episode terminates at $t_K = T$ (fixed horizon) or on
        numerical divergence of the dynamics.
\end{itemize}
SAC is selected because its entropy regularisation is well-matched to
the non-stationarity of the meta-MDP (the distribution of $z(t)$
conditional on the history changes across episodes and across the
payload). The alternative of parameter-level TRPO or PPO was
explored in pilot experiments and produced less stable training on
this task.

\paragraph{Shielded admissibility.}
The admissibility constraint of Section~\ref{sec:theory} is enforced
via the Lagrangian reformulation
\begin{equation}
  \min_\theta
    \mathbb{E}\!\left[ \int_0^T \ell(e, \dot e)\, dt \right]
  \;+\;
  \beta \,
    \mathbb{E}\!\left[ \mathrm{dist}\bigl( G_\theta(c(t)),\, \Pi_{\mathrm{adm}}(x(t)) \bigr) \right],
\end{equation}
with dual multiplier $\beta$ tuned by the standard Lagrangian update
rule of~\cite{Altman1999CMDP}. At runtime, a projection onto
$\Pi_{\mathrm{adm}}(x(t))$ is applied as a final safety net
(Theorem~\ref{thm:stab}(b)); at convergence of the Lagrangian, the
projection is inactive
(Proposition~\ref{prop:shield}).

\subsection{Summary of architectural degrees of freedom}
\label{sec:meta:summary}

The meta-controller thus has the following architectural parameters:
\begin{itemize}
  \item window $W$ (Section~\ref{sec:meta:window});
  \item attention head count $K$ and per-head dimension $d_k$,
        jointly determining the model dimension
        $d_{\mathrm{model}} = K \cdot d_k$;
  \item number of stacked attention layers $L$;
  \item presence of a per-token feed-forward block after each
        attention layer;
  \item width of the subsequent MLP head
        (fixed at $[64, 64]$ throughout).
\end{itemize}
Of these, $K$ is fixed from Phase~1 INCRT as described in
Section~\ref{sec:incrt}. The remaining four hyperparameters
$(W, d_k, L, \mathrm{FFN})$ are ablated in Section~\ref{sec:experiments}
to characterise the contribution of each to tracking performance
across the memory-horizon range $\tau_z \in \{1, 2, 5\}$~s.


%

\section{Theoretical Framework}
\label{sec:theory}

This section develops the four results that underpin the proposed
meta-controller. Sections~\ref{sec:theory:admissible}--\ref{sec:theory:shield}
extend the shielded safe-RL framework of the parent
work~\cite{CirrincioneFagiolini2026} from action-level to parameter-level control.
Section~\ref{sec:theory:markov} establishes a quantitative lower bound on
the tracking error attainable by any Markovian meta-controller, thereby
motivating the use of windowed temporal attention.
Section~\ref{sec:theory:headcount} specialises the INCRT head-count
analysis of~\cite{CirrincioneINCRT2026} to the temporal residual
operator relevant for meta-control, and closes the theoretical loop by
showing that the window size $W$ and the head count $K$ can be
independently selected from task-specific quantities.

\subsection{Admissible parameter set}
\label{sec:theory:admissible}

Let $V_\psi: \mathbb{R}^{5n} \to \mathbb{R}_{\geq 0}$ be the learned
Lyapunov function inherited from~\cite{CirrincioneFagiolini2026}: structured
quadratic, positive-definite outside the target, and $L_\psi$-Lipschitz
in $\nabla V_\psi$. Fix a desired decay rate $\alpha > 0$.

\begin{definition}[Admissible parameter set]
\label{def:admissible}
For every state $x = (q, \dot q, e, \dot e, s) \in \mathbb{R}^{5n}$,
\begin{equation}
  \Pi_{\mathrm{adm}}(x) \;:=\;
  \bigl\{\,\theta_{\mathrm{ctrl}} \in \mathcal{P} \,:\,
    \dot V_\psi(x; \theta_{\mathrm{ctrl}}) + \alpha V_\psi(x) \leq 0
  \,\bigr\},
\end{equation}
where $\dot V_\psi(x; \theta_{\mathrm{ctrl}})$ denotes the instantaneous
Lyapunov rate along the closed-loop trajectory induced by the
computed-torque (CT) law with parameters
$\theta_{\mathrm{ctrl}} = (K_d, \Lambda, \eta)$, and $\mathcal{P}$ is a
compact box in $\mathbb{R}^m$ enforcing the operating ranges of each
gain and feed-forward weight.
\end{definition}

Writing $\tau = \mathrm{CT}(\theta_{\mathrm{ctrl}}) + \phi_{\mathrm{ff}}(\eta)$
and using the affine--in--$\tau$ structure of the shield inequality
derived in~\cite[Sec.~IV]{CirrincioneFagiolini2026},
\[
  b(x)^\top \tau \;\leq\; c(x),
  \qquad
  b(x) := \nabla_{\dot q} V_\psi \cdot g(x), \quad
  c(x) := -\alpha V_\psi(x) - \nabla V_\psi \cdot f(x),
\]
the CT contribution expands as
\begin{align}
  b(x)^\top \mathrm{CT}(q,\dot q, q_d; K_d, \Lambda)
    &= K_d^\top A_1(x) + \Lambda^\top A_2(x) + a_0(x),
  \label{eq:ct_affine} \\
  b(x)^\top \phi_{\mathrm{ff}}(q,\dot q; \eta)
    &= \eta^\top B(x) + b_0(x),
  \label{eq:ff_affine}
\end{align}
with $A_1, A_2, B$ state-dependent vectors and $a_0, b_0$ state-dependent
scalars.

\begin{lemma}[Convexity of $\Pi_{\mathrm{adm}}(x)$]
\label{lem:convex}
Fix $x \in \mathbb{R}^{5n}$. If the feed-forward map
$\phi_{\mathrm{ff}}(q,\dot q; \eta)$ is affine in $\eta$ --- that is,
$\phi_{\mathrm{ff}}(q, \dot q; \eta) = \Phi(q, \dot q)\, \eta$ for a
state-dependent feature matrix $\Phi$ --- then $\Pi_{\mathrm{adm}}(x)$
is a (possibly unbounded) convex polytope: the intersection of the box
$\mathcal{P}$ with the single affine half-space
\begin{equation}
  K_d^\top A_1(x) + \Lambda^\top A_2(x) + \eta^\top B(x)
  + \bigl( a_0(x) + b_0(x) \bigr) \;\leq\; c(x).
\end{equation}
\end{lemma}

\begin{proof}
The constraint defining $\Pi_{\mathrm{adm}}(x)$ is affine in
$\theta_{\mathrm{ctrl}} = (K_d, \Lambda, \eta)$ by \eqref{eq:ct_affine}
and \eqref{eq:ff_affine}; the intersection of a half-space with the
compact convex box $\mathcal{P}$ is convex.
\end{proof}

\begin{remark}[Nonlinear feed-forward]
If $\phi_{\mathrm{ff}}$ is nonlinear in $\eta$, the half-space becomes a
smooth sublevel set of a smooth function; convexity may fail. In
practice we restrict to the affine class, which forfeits no
expressiveness at the meta-controller level since the feature matrix
$\Phi(q,\dot q)$ can itself be the output of a deep network.
\end{remark}

\begin{assumption}[Non-emptiness]
\label{ass:nonempty}
For every $x$ in the operating set $\mathcal{X} \subset \mathbb{R}^{5n}$,
$\Pi_{\mathrm{adm}}(x) \cap \mathcal{P} \neq \emptyset$.
\end{assumption}

This is the meta-level analogue of the drift-decay condition on the
control-degeneracy set $\mathcal{Z}$ in~\cite{CirrincioneFagiolini2026}. It
asserts that at every state there exists at least one gain choice
compatible with the prescribed Lyapunov decrease.

\subsection{Stability under parameter modulation}
\label{sec:theory:stability}

Let $G_\theta: \mathbb{R}^{n_c} \to \mathcal{P}$ be the learned
meta-controller mapping a context token $c(t) \in \mathbb{R}^{n_c}$ to a
parameter tuple $\theta_{\mathrm{ctrl}}(t)$.

\begin{assumption}[Bounded context velocity]
\label{ass:ctxvel}
The context token $c(t)$ is piecewise-$C^1$ and satisfies
$\| \dot c(t) \| \leq V_c$ for all $t \geq 0$.
\end{assumption}

Under the chain rule,
\begin{equation}
  \| \dot \theta_{\mathrm{ctrl}}(t) \| \;\leq\; L_{G_\theta} \| \dot c(t) \|
  \;\leq\; L_{G_\theta} V_c,
  \label{eq:theta_dot_bound}
\end{equation}
where $L_{G_\theta}$ is the Lipschitz constant of $G_\theta$ on a
compact set containing the context trajectory.

\begin{theorem}[Stability under parameter modulation]
\label{thm:stab}
Let $G_\theta$ be $L_{G_\theta}$-Lipschitz, let
Assumption~\ref{ass:nonempty} hold, and let
Assumption~\ref{ass:ctxvel} hold. Suppose either
\begin{enumerate}[label=(\alph*)]
  \item $G_\theta(c(t)) \in \Pi_{\mathrm{adm}}(x(t))$ for all $t \geq 0$, or
  \item a runtime projection
    $\theta_{\mathrm{ctrl}}(t) = \Pi_{\Pi_{\mathrm{adm}}(x(t))}\bigl( G_\theta(c(t)) \bigr)$
    is applied.
\end{enumerate}
Then the closed-loop system satisfies
$\dot V_\psi(x(t)) + \alpha V_\psi(x(t)) \leq 0$ for all $t$, hence
$V_\psi(x(t)) \leq V_\psi(x(0))\, e^{-\alpha t}$, and the tracking
error $\| e(t) \|$ decays exponentially at rate $\alpha/2$.
\end{theorem}

\begin{proof}
Under (a), the definition of $\Pi_{\mathrm{adm}}$ directly gives
$\dot V_\psi + \alpha V_\psi \leq 0$. Under (b), the projection onto
$\Pi_{\mathrm{adm}}(x(t))$ is well-defined (by Lemma~\ref{lem:convex})
and feasible (by Assumption~\ref{ass:nonempty}), and preserves the
inequality by construction. In either case,
$V_\psi(t) \leq V_\psi(0)\, e^{-\alpha t}$. Since $V_\psi$ is quadratic
in $e$ and bounded below by $\lambda_{\min} \| e \|^2$ for some
$\lambda_{\min} > 0$, the tracking error decays at rate $\alpha/2$.
\end{proof}

\begin{remark}[Two regimes]
Case~(a) corresponds to the \emph{training-only} regime: if $G_\theta$
has been trained to respect the admissibility constraint, no runtime
machinery is needed. Case~(b) is the \emph{runtime-shielded} regime: 
$G_\theta$ is trained unconstrained and projected at runtime. In
practice we combine the two via a Lagrangian training objective plus a
runtime projection as a safety net.
\end{remark}

\subsection{Shield activation at the optimum}
\label{sec:theory:shield}

A distinctive finding of the 7-DOF scalability study
in~\cite{CirrincioneFagiolini2026} was shield activation reaching $\sim 70\%$
when the policy and shield were trained as adversaries. The
meta-controller formulation corrects this.

\begin{proposition}[Vanishing shield activation at the optimum]
\label{prop:shield}
Let $\widehat G_\theta$ be the optimum of the constrained Lagrangian
objective
\begin{equation}
  \min_{\theta}\;
  \mathbb{E}_{\tau \sim \pi_\theta}\!\!
    \left[ \int_0^T \ell(q, q_d)\, dt \right]
  \quad \text{subject to} \quad
  G_\theta(c(t)) \in \Pi_{\mathrm{adm}}(x(t))
  \;\; \text{almost surely}.
\end{equation}
Then the shield activation fraction at $\widehat G_\theta$ is zero.
In a finite-training regime, the shield activation fraction is
bounded by
\begin{equation}
  P_{\mathrm{shield}} \;\leq\; C \cdot
  \mathbb{E} \bigl[ \mathrm{dist}\bigl( G_\theta(c),\, \Pi_{\mathrm{adm}}(x) \bigr) \bigr],
  \label{eq:shield_bound}
\end{equation}
which vanishes as the Lagrangian training converges.
\end{proposition}

\begin{proof}
The shield fires if and only if the policy output is infeasible.
Under Slater's condition, which holds by Assumption~\ref{ass:nonempty},
strong duality gives almost-sure constraint satisfaction at the
optimum. In a finite-training regime, the feasibility gap is dominated
by the expected distance of the policy output from the feasible set,
yielding \eqref{eq:shield_bound}; $C$ depends on the Lipschitz constants
of $b(x)$ and $c(x)$ in Lemma~\ref{lem:convex}.
\end{proof}

\begin{remark}[Interpretation]
The parent framework~\cite{CirrincioneFagiolini2026} treats shield activation as
an emergent quantity with no predicted value; the meta-controller
framework makes it predictable and small by design. The $\sim 70\%$
activation rate observed in the 7-DOF study is thereby attributed to
the adversarial training regime, not to a structural limitation.
\end{remark}

\subsection{The Markovian optimality gap}
\label{sec:theory:markov}

We now establish the core result motivating the use of windowed
temporal attention.

\begin{definition}[System with unobservable memory]
\label{def:unobs_mem}
An Euler--Lagrange system has \emph{unobservable memory} if there
exists an internal state $z(t) \in \mathbb{R}^{n_z}$ with dynamics
$\dot z = \zeta(q, \dot q, z)$ such that:
\begin{enumerate}[label=(\roman*)]
  \item $z(t)$ is not directly measured by the sensor suite;
  \item the steady-state value of $z$ given $(q, \dot q)$ is not a
        function of $(q, \dot q)$ alone, but depends on the history
        $\{ q(s), \dot q(s) : s \leq t \}$;
  \item the memory horizon
        $H_z := \|\partial \zeta / \partial z\|^{-1}$ is finite.
\end{enumerate}
\end{definition}

\begin{example}[Stribeck friction]
The Stribeck friction model $F_s(\dot q, z)$ with
$\dot z = -z/\tau_z + \lambda_z \dot q$ satisfies
Definition~\ref{def:unobs_mem} with memory horizon $H_z = \tau_z$.
\end{example}

Denote by $\theta^\star_{\mathrm{ctrl}}(t)$ the time-optimal parameter
trajectory --- the one that minimises expected tracking error over a
realisation of the task distribution --- and by $\ell^\star$ the
associated expected cost. For a meta-controller $G$, let $\ell(G)$
denote its expected cost.

The next result quantifies, under regularity assumptions on the
cost, the price paid by a Markovian meta-controller that ignores
the unobservable memory. The result is stated in \emph{informal}
form: the hypotheses express the regularity required for the
argument to go through and are not minimal in the structural
sense. A more general formulation, removing the strong-convexity
assumption or weakening the observability assumption on $z(t)$,
would require substantially more machinery than is developed
here and is left to future work.

\begin{proposition}[Markovian optimality gap, informal]
\label{prop:markov}
Let the system satisfy Definition~\ref{def:unobs_mem} with memory
horizon $H_z$, and suppose the following regularity conditions
hold:
\begin{enumerate}[label=(R\arabic*),leftmargin=2.2em]
  \item the instantaneous cost
        $\ell(\theta_{\mathrm{ctrl}}; q, \dot q, z)$ is twice
        continuously differentiable and strongly convex in
        $\theta_{\mathrm{ctrl}}$ with modulus $\mu > 0$;
  \item the sensitivity
        $\kappa = \|\partial^2 \ell / \partial z \, \partial
        \theta_{\mathrm{ctrl}}\|$ is bounded and positive on a
        neighbourhood of the reference trajectory;
  \item the conditional variance
        $\sigma_z^2 = \mathbb{E}[\,\mathrm{Var}(z(t)\,|\,q(t),
        \dot q(t))\,]$ is finite and strictly positive.
\end{enumerate}
Then:
\begin{enumerate}[label=(\roman*)]
  \item The time-optimal parameter trajectory depends on the full
        history:
        \begin{equation}
          \theta^\star_{\mathrm{ctrl}}(t)
          = \theta^\star_{\mathrm{ctrl}}\bigl( \{q(s), \dot q(s), z(s) : s \leq t\} \bigr),
          \label{eq:theta_star}
        \end{equation}
        and is not a function of $(q(t), \dot q(t))$ alone.
  \item Any Markovian meta-controller
        $G^{\mathrm{Mk}}: (q, \dot q) \mapsto \mathcal{P}$
        incurs an expected excess cost bounded below by
        \begin{equation}
          \mathbb{E}[\ell(G^{\mathrm{Mk}})] - \ell^\star
          \;\geq\; c_1 \, \sigma_z^2,
          \label{eq:markov_bound}
        \end{equation}
        where
        $\sigma_z^2 = \mathbb{E}[\,\mathrm{Var}(z(t)\,|\,q(t), \dot q(t))\,]$
        is the conditional variance of $z$ given the current state, and
        $c_1 = \mu \, \kappa^2 / 2$ with $\kappa$ the sensitivity of the
        cost to $z$.
  \item A windowed meta-controller
        $G^W: (q, \dot q, \text{history of length } W) \mapsto \mathcal{P}$
        with $W \geq H_z$ and sufficient representational capacity can
        reduce the excess cost to $o(\sigma_z^2)$.
\end{enumerate}
\end{proposition}

\begin{proof}[Proof sketch]
Part (i) follows from the hidden-state identifiability principle for
systems with unobservable memory: since $z(t)$ enters the optimal
control law through the state equation and $z(t)$ is not
$\sigma(q,\dot q)$-measurable by Definition~\ref{def:unobs_mem}(ii),
neither is the optimiser.
Part (ii) is an application of the information bottleneck argument:
the output of $G^{\mathrm{Mk}}$ is
$\sigma(q,\dot q)$-measurable, which is strictly coarser than
$\sigma(q, \dot q, z)$. Strong convexity of the cost with modulus
$\mu$, combined with the conditional variance identity
$\mathrm{Var}(\theta^\star\,|\,q,\dot q) \geq \kappa^2 \sigma_z^2$
(where $\kappa = \|\partial \theta^\star / \partial z\|$), yields
\eqref{eq:markov_bound}.
Part (iii) follows because the joint process
$\{(q(s), \dot q(s)) : s \in [t-W, t]\}$ is sufficient for $z(t)$
whenever $W \geq H_z$: in particular, the identifiability map
$\{q(s), \dot q(s)\}_{s \in [t-W,t]} \mapsto z(t)$ is measurable, and a
universal approximator with attention over the window can implement
it. The full argument is given in Appendix~\ref{app:markov}.
\end{proof}

\begin{remark}[Scaling of $\sigma_z^2$ with $\tau_z$]
\label{rem:sigma_scaling}
For linear unobservable-memory dynamics
$\dot z = -z/\tau_z + \lambda_z \dot q$ driven by a zero-mean
second-order stationary process $\dot q$, the stationary
conditional variance $\sigma_z^2$ is monotonically increasing in
$\tau_z$:
\begin{equation}
  \sigma_z^2(\tau_z)
  = \lambda_z^2 \cdot
    \frac{\tau_z}{2} \cdot
    \bigl( 1 - \rho_{\dot q}(\tau_z) \bigr),
\end{equation}
where $\rho_{\dot q}(\tau)$ is the autocorrelation of $\dot q$ at lag
$\tau$. In particular, $\sigma_z^2 \to 0$ as $\tau_z \to 0$
(instantaneous memory) and $\sigma_z^2 \to \lambda_z^2 \, \mathbb{E}[\dot q^2] / 2$
as $\tau_z \to \infty$ (long memory saturation). The derivation is
given in Appendix~\ref{app:sigma_scaling}.
\end{remark}

\begin{corollary}[Empirical implication]
\label{cor:crossover}
Proposition~\ref{prop:markov} predicts a regime-dependent ordering of
meta-controllers: at $\tau_z \to 0$, a memoryless policy is optimal
($\sigma_z^2 \to 0$, Markovian gap vanishes); at moderate $\tau_z$,
memoryless policies degrade at rate $c_1 \, \sigma_z^2$; at
$\tau_z \gtrsim W \cdot \Delta t$, windowed attention recovers optimal
performance. The crossover point is operationally testable and is
reported in Section~\ref{sec:experiments}.
\end{corollary}

\subsection{Temporal head-count bound}
\label{sec:theory:headcount}

Proposition~\ref{prop:markov}(iii) establishes that a windowed
meta-controller with sufficient capacity closes the Markovian gap.
What remains is to quantify ``sufficient capacity''.

The INCRT framework of~\cite{CirrincioneINCRT2026} provides a
principled determination of the required number of attention heads
through a residual-matrix construction. We briefly recall the relevant
elements before stating the adaptation to the temporal setting.

\paragraph{INCRT residual operator.}
Given a training signal represented as a matrix
$Y \in \mathbb{R}^{W \times d_{\mathrm{task}}}$, INCRT constructs a
residual operator $A_{\mathrm{res}} \in \mathbb{R}^{d \times d}$ whose
spectrum determines the head count required to represent
$Y$~\cite[Thm.~7]{CirrincioneINCRT2026}. Under three structural
properties of $A_{\mathrm{res}}$ --- symmetry, rank-one deflation, and
incoherence of the deflated directions --- Theorem~7 of the INCRT paper
gives the compressed-sensing upper bound on the head count necessary
to reconstruct the signal up to an $\epsilon$-residual.

\paragraph{Temporal adaptation.}
In the meta-control setting, the relevant residual operator acts on
the \emph{temporal} domain of the history window rather than on the
feature domain. Specifically, define
\begin{equation}
  A_{\mathrm{res}}^{\mathrm{temp}}
  = \mathbb{E}\!\left[
      \nabla_{\mathrm{hist}} z(t)\, \nabla_{\mathrm{hist}} z(t)^\top
    \right]
  \;\in\; \mathbb{R}^{W \times W},
  \label{eq:A_res_temp}
\end{equation}
the auto-covariance matrix of the gradient of $z(t)$ with respect to
the history of lookbacks.

The next result imports a compressed-sensing bound from the
static INCRT framework to the temporal operator. It is stated
in \emph{informal} form: the three structural hypotheses
(symmetry, rank-one deflation, and incoherence) are exactly the
hypotheses of \cite[Thm.~7]{CirrincioneINCRT2026}, and are
verified for $A_{\mathrm{res}}^{\mathrm{temp}}$ under the
regularity conditions (R1)--(R3) and the mild additional
assumption that the autocorrelation of $\dot q$ does not vanish
identically in the window. A full proof with minimal
hypotheses is beyond the scope of the present paper.

\begin{proposition}[Temporal head-count bound, informal]
\label{prop:headcount}
Under the regularity conditions of Proposition~\ref{prop:markov}
and the assumption that the autocorrelation of $\dot q$ on the
window interval $[-W\Delta t,\,0]$ is strictly positive at
lag zero and does not vanish identically at any other lag, the
operator $A_{\mathrm{res}}^{\mathrm{temp}}$ defined by
\eqref{eq:A_res_temp} satisfies the three structural properties
(symmetry, rank-one deflation, and incoherence) required
by~\cite[Thm.~7]{CirrincioneINCRT2026}. Consequently, the optimal head
count for the temporal attention operator is upper-bounded by
\begin{equation}
  K^\star(\tau_z)
  \;\leq\;
  C \cdot r_{\mathrm{eff}}\bigl( A_{\mathrm{res}}^{\mathrm{temp}}(\tau_z) \bigr),
  \label{eq:K_star_bound}
\end{equation}
where $r_{\mathrm{eff}}(\cdot)$ is the effective rank
(stable rank) and $C$ is a universal constant inherited from
\cite[Thm.~7]{CirrincioneINCRT2026}.
\end{proposition}

\begin{proof}[Proof sketch]
Symmetry of $A_{\mathrm{res}}^{\mathrm{temp}}$ is immediate from its
definition as a covariance matrix. Rank-one deflation follows from the
fact that the gradient of $z(t)$ with respect to the history is
dominated by a single scalar multiplier at each lag (the coefficient
$\lambda_z$ of the $\dot q$ term in $\dot z$). Incoherence of the
deflated directions follows from the non-degenerate autocorrelation of
$\dot q$ away from zero lag. The three properties are verified in
detail in Appendix~\ref{app:headcount}, and the compressed-sensing
bound of~\cite[Thm.~7]{CirrincioneINCRT2026} applies directly.
\end{proof}

\begin{corollary}[Non-monotonic $K^\star(\tau_z)$]
\label{cor:nonmonotonic}
The effective rank
$r_{\mathrm{eff}}(A_{\mathrm{res}}^{\mathrm{temp}}(\tau_z))$ is
non-monotonic in $\tau_z$: it increases with $\tau_z$ in the regime
$\tau_z \lesssim W \cdot \Delta t$ (more of the memory horizon fits
inside the window, more temporal modes can be separated) and decreases
with $\tau_z$ in the regime $\tau_z \gtrsim W \cdot \Delta t$ (the
window saturates, only low-frequency modes remain distinguishable).
The maximum of $r_{\mathrm{eff}}$ occurs near
$\tau_z \approx W \cdot \Delta t / 2$. Proof in
Appendix~\ref{app:nonmonotonic}.
\end{corollary}

\paragraph{Design implication.}
Propositions~\ref{prop:markov} and~\ref{prop:headcount} give two
independent ingredients for selecting the meta-controller architecture:
\begin{itemize}
  \item the \emph{window size} $W$ is selected as
        $W \geq H_z / \Delta t$, so that the window covers the memory
        horizon (Proposition~\ref{prop:markov}(iii));
  \item the \emph{head count} $K$ is selected as
        $K \leq C \cdot r_{\mathrm{eff}}(A_{\mathrm{res}}^{\mathrm{temp}})$,
        which is non-monotonic in $\tau_z$ with a peak near $\tau_z \approx W \cdot \Delta t / 2$
        (Proposition~\ref{prop:headcount} and Corollary~\ref{cor:nonmonotonic}).
\end{itemize}
The two selections are decoupled: $W$ is determined by the memory
horizon of the system, $K$ by the rank structure of the temporal
residual operator. Section~\ref{sec:incrt} operationalises the
selection of $K$ via the Phase~1 INCRT procedure, and
Section~\ref{sec:experiments} validates the decoupling empirically
across a grid of $(K, d_k, L, W)$ combinations.

\subsection{Summary of the theoretical roadmap}
\label{sec:theory:summary}

The four results established in this section assemble into a single
design prescription.
Lemma~\ref{lem:convex} ensures that the admissible parameter set is a
convex polytope and that the runtime projection of
Theorem~\ref{thm:stab}(b) is well-defined.
Theorem~\ref{thm:stab} guarantees closed-loop stability under any
Lipschitz meta-controller whose output lies in (or is projected onto)
the admissible set.
Proposition~\ref{prop:shield} predicts that shield activation vanishes
at the optimum of the constrained Lagrangian, resolving the
$\sim 70\%$ activation anomaly of the parent framework.
Proposition~\ref{prop:markov} establishes the lower bound $c_1 \sigma_z^2$
on the Markovian-policy error and the upper bound $o(\sigma_z^2)$
on the windowed-policy error, with $\sigma_z^2$ growing monotonically
in $\tau_z$.
Proposition~\ref{prop:headcount} provides the head-count bound for the
temporal attention operator, and Corollary~\ref{cor:nonmonotonic}
predicts non-monotonic $K^\star(\tau_z)$ with a peak near the
window-to-horizon matched regime.
The experimental section tests each prediction in turn.


%

\section{INCRT for architecture determination}
\label{sec:incrt}

Proposition~\ref{prop:headcount} leaves unanswered a practical
question: given a system with memory horizon $\tau_z$, how is the
head-count bound $K^\star(\tau_z)$ computed so as to feed into the
design of the Phase~2 meta-controller? The answer developed in
this section decomposes into two stages, rendered schematically in
Figure~\ref{fig:pipeline}. In the first stage, run offline and
on a CPU, the trajectory simulator produces short rollouts at the
target horizon; a temporal residual operator is assembled from
these rollouts, its effective rank is estimated by the
incremental rank-tracking procedure of~\cite{CirrincioneINCRT2026},
and a single integer $K^\star$ is returned. In the second stage,
run on a GPU for the reinforcement-learning training, the integer
$K^\star$ is used to bound a short grid search over the head count
of the attention block; the best head count on the grid is then
retained and the corresponding policy is trained to convergence
under the shielded admissibility constraint of
Section~\ref{sec:meta:sac}. The first stage is the
architecture-selection phase and takes seconds to minutes; the
second stage is the policy-training phase and takes hours.
Architecture selection and policy optimisation are therefore
decoupled, and the end product is a meta-controller with no
runtime growing or pruning machinery. The separation follows the
search-then-retrain pattern common in neural architecture
search~\cite{Liu2019DARTS,Pham2018ENAS,Miao2022RLDARTS}.

\begin{figure}[t]
  \centering
  \includegraphics[width=0.95\linewidth]{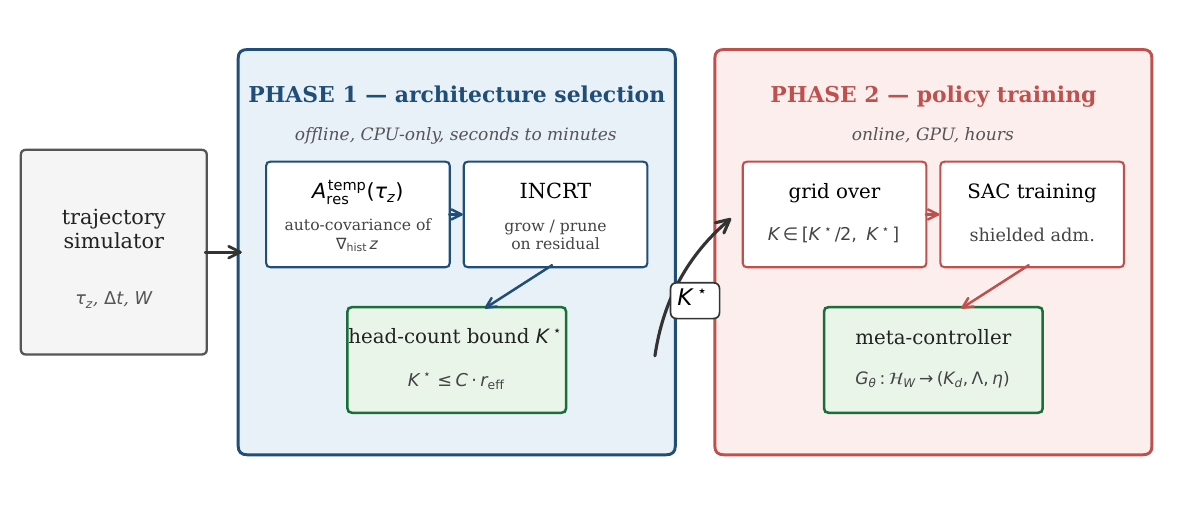}
  \caption{Two-phase architecture-selection pipeline. Phase~1
           returns a head-count bound $K^\star$ from an offline
           analysis of the memory structure; Phase~2 runs a
           constrained grid search over $K$ and trains the
           selected policy.}
  \label{fig:pipeline}
\end{figure}

\subsection{The INCRT framework}
\label{sec:incrt:framework}

INCRT~\cite{CirrincioneINCRT2026} is an incremental construction of a
sparse signal decomposition based on a residual matrix that evolves
across growth and pruning iterations. We summarise only the elements
needed in the sequel; a full development is available in the
reference.

\paragraph{Residual matrix.}
Given a training signal $Y \in \mathbb{R}^{n \times d}$ and a
candidate rank-$K$ decomposition
$Y \approx \sum_{k=1}^{K} u_k v_k^\top$ with
$u_k \in \mathbb{R}^n$, $v_k \in \mathbb{R}^d$, the
\emph{residual matrix} is
\begin{equation}
  A_{\mathrm{res}}^{(K)}
  = Y Y^\top - \sum_{k=1}^{K} u_k v_k^\top \cdot (u_k v_k^\top)^\top.
\end{equation}
The effective rank $r_{\mathrm{eff}}(A_{\mathrm{res}}^{(K)})$
measures the residual spectral mass not yet explained by the first $K$
components.

\paragraph{Growth signal.}
At each iteration $t$, INCRT proposes a candidate new direction
$(u_{K+1}, v_{K+1})$ obtained from the leading eigenvector of the
current residual. The \emph{growth signal} quantifies the reduction in
effective rank that would result from adding the candidate:
\begin{equation}
  g_{K+1}^{(t)}
  = r_{\mathrm{eff}}(A_{\mathrm{res}}^{(K)})
    - r_{\mathrm{eff}}(A_{\mathrm{res}}^{(K+1)}).
\end{equation}
When $g_{K+1}^{(t)}$ exceeds a prescribed threshold
$\gamma_{\mathrm{add}}$, the direction is accepted and $K \gets K+1$.

\paragraph{Pruning.}
Symmetrically, if the contribution of a previously admitted direction
falls below a threshold $\gamma_{\mathrm{prune}}$, the direction is
removed and $K \gets K-1$. Growth and pruning are complementary
operations on a single homeostatic loop.

\paragraph{Bidirectional gate.}
The gate mechanism of~\cite[Sec.~4]{CirrincioneINCRT2026} ensures
numerical stability of the growth--pruning alternation: directions
that oscillate across the growth and pruning thresholds are stabilised
by an averaging filter with time constant governed by the
NTK-alignment rate at the current $K$.

\paragraph{Homeostatic convergence.}
\cite[Thm.~6]{CirrincioneINCRT2026} establishes that the joint
dynamics converge to a fixed-point rank $K^\star$ characterised by the
effective rank of $A_{\mathrm{res}}$: no candidate growth exceeds
$\gamma_{\mathrm{add}}$ and no candidate pruning exceeds
$\gamma_{\mathrm{prune}}$ at $K = K^\star$.

\paragraph{Compressed-sensing bound.}
\cite[Thm.~7]{CirrincioneINCRT2026} bounds $K^\star$ from above by a
compressed-sensing inequality applied to $A_{\mathrm{res}}$ under
three structural hypotheses: symmetry, rank-one deflation, and
incoherence of the deflated directions. When these hypotheses hold,
\begin{equation}
  K^\star \;\leq\; C \cdot r_{\mathrm{eff}}\bigl( A_{\mathrm{res}} \bigr),
\end{equation}
where $C$ is a universal constant. This is the inequality we exploit
in the temporal adaptation below.

\subsection{Temporal residual operator}
\label{sec:incrt:temporal}

The residual matrix of INCRT is constructed over the
\emph{feature domain} of the training signal $Y$. For meta-control
with unobservable memory, however, the quantity to be recovered is not
a feature--space latent but a function of the \emph{temporal history}
of observations: specifically, $z(t)$ as a map of the past $W$
observation steps.
The natural adaptation is therefore to construct a residual operator
over the time domain rather than the feature domain.

\paragraph{Design choices.}
Three candidate constructions were considered:
\begin{description}
  \item[(v1) Feature-space residual.]
    Direct application of the standard INCRT construction with $Y$
    taken as the observation feature matrix over a batch of histories.
    This yields $K^\star$ approximately constant (between $3$ and $5$)
    across the range $\tau_z \in [0.2, 5]$s --- the feature dimension
    does not encode the memory structure of the task.
  \item[(v2) Multi-horizon prediction target.]
    $Y$ is replaced with a stack of future-horizon predictions of the
    observation, horizon-indexed. This exposes some temporal structure
    but $K^\star(\tau_z)$ remains flat at $\approx 5$ throughout the
    range: the future-horizon stack dilutes the memory dependence.
  \item[(v3) Temporal residual operator.]
    $Y$ is replaced with a history matrix indexed by temporal lag.
    The residual operator becomes
    $A_{\mathrm{res}}^{\mathrm{temp}} \in \mathbb{R}^{W \times W}$,
    and $K^\star(\tau_z)$ varies non-monotonically with $\tau_z$ as
    predicted by Corollary~\ref{cor:nonmonotonic}.
\end{description}
Only the third construction, developed in detail below, exhibits the
expected regime dependence; the first two are not pursued further.

\paragraph{Construction of $A_{\mathrm{res}}^{\mathrm{temp}}$.}
Let $\mathcal{H}_W(t) = (q(t), q(t-1), \ldots, q(t - W + 1),
\dot q(t), \ldots, \dot q(t - W + 1))$ be the history window at
time $t$. Define the temporal gradient of the memory state:
\begin{equation}
  \nabla_{\mathrm{hist}}\, z(t) \;:=\;
  \frac{\partial z(t)}{\partial \mathcal{H}_W(t)}
  \;\in\; \mathbb{R}^W.
\end{equation}
The temporal residual operator is the auto-covariance of this gradient,
evaluated over a training set of trajectories sampled at the target
$\tau_z$:
\begin{equation}
  A_{\mathrm{res}}^{\mathrm{temp}}(\tau_z)
  \;=\;
  \mathbb{E}\!\left[
    \nabla_{\mathrm{hist}}\, z(t)\, \nabla_{\mathrm{hist}}\, z(t)^\top
  \right].
\end{equation}
By Proposition~\ref{prop:headcount} the three structural hypotheses of
\cite[Thm.~7]{CirrincioneINCRT2026} are satisfied, so the INCRT growth
and pruning dynamics apply to $A_{\mathrm{res}}^{\mathrm{temp}}$
without modification, and the resulting $K^\star(\tau_z)$ is
compressed-sensing-bounded by the effective rank of the operator.

\subsection{Phase~1 / Phase~2 separation}
\label{sec:incrt:separation}

We now describe the architectural role of INCRT in the present paper,
which is different from, and more conservative than, the use in the
original JMLR work.

\paragraph{Phase~1: architecture search.}
The INCRT dynamics run on the surrogate temporal-representation task
defined by $A_{\mathrm{res}}^{\mathrm{temp}}$. The task is supervised
and computationally inexpensive compared to SAC training: it requires
only the ability to simulate $z(t)$ from the system dynamics and to
sample histories $\mathcal{H}_W(t)$. The output of Phase~1 is the
converged head count $K^\star(\tau_z)$ for each $\tau_z$ of interest.

\paragraph{Phase~2: fixed-architecture policy training.}
The Phase~2 meta-controller is a standard attention-based feature
extractor with $K^\star$ heads --- \emph{fixed} for the duration of
the SAC training. There is no growth or pruning during Phase~2:
$K^\star$ is an architectural hyperparameter, not a dynamic variable.

\paragraph{Why search-then-retrain.}
Three reasons motivate this separation:
\begin{enumerate}[label=(\roman*)]
  \item \emph{Optimisation disentanglement.} Running INCRT dynamics
        concurrently with SAC training would require the SAC critic
        to track a non-stationary policy-parameter dimension, a
        known source of training instabilities in actor-critic RL
        even when the architecture change is slow.
  \item \emph{Deployment.} A fixed-architecture Phase~2 yields a
        meta-controller with a specified parameter count, latency
        profile, and memory footprint, all determined at design time.
  \item \emph{Empirical testability.} The surrogate-optimal $K^\star$
        may or may not coincide with the RL-optimal head count in
        Phase~2; separating the two stages makes this difference
        empirically tractable (Section~\ref{sec:experiments}).
\end{enumerate}
This pattern aligns with the practice of differentiable neural
architecture search~\cite{Liu2019DARTS,Pham2018ENAS}: a search stage on a surrogate
task produces a discrete architecture, which is then trained from
scratch on the target task.

\paragraph{Scope note on INCRT usage.}
The present paper uses INCRT exclusively in its architecture-search
role. The broader homeostatic-dynamics interpretation
of~\cite{CirrincioneINCRT2026}, including its connection to
representational completeness in deep
architectures~\cite{CirrincioneAIJ2026}, is not invoked here. An
extension that runs INCRT directly inside Phase~2 --- effectively a
dynamic-architecture meta-controller --- is a natural follow-up
indicated in Section~\ref{sec:discussion}.

\subsection{Phase~1 surrogate protocol}
\label{sec:incrt:protocol}

Algorithm~\ref{alg:phase1} summarises the Phase~1 procedure used in
the present work.

\begin{algorithm}[t]
\caption{Phase~1 INCRT architecture search for temporal meta-control}
\label{alg:phase1}
\begin{algorithmic}[1]
\Require Memory horizon $\tau_z$; window $W$; growth / pruning
         thresholds $\gamma_{\mathrm{add}}, \gamma_{\mathrm{prune}}$;
         trajectory simulator; sample size $N$.
\Ensure Converged head count $K^\star(\tau_z)$.
\State Sample $N$ trajectories of the Euler--Lagrange system at
       horizon $\tau_z$; compute $\mathcal{H}_W(t)$ and $z(t)$
       along each trajectory.
\State Form the temporal residual operator
       $\widehat A_{\mathrm{res}}^{\mathrm{temp}}(\tau_z)$ as the
       empirical auto-covariance of $\nabla_{\mathrm{hist}} z$ over
       the $N$ samples.
\State Initialise $K \gets 1$; $U \gets (u_1)$ with $u_1$ the
       leading eigenvector of $\widehat A_{\mathrm{res}}^{\mathrm{temp}}$;
       residual $R \gets \widehat A_{\mathrm{res}}^{\mathrm{temp}}$.
\Repeat
  \State \emph{Growth step.}
    Propose $u_{K+1}$ as the leading eigenvector of $R$.
    Compute the growth signal
    $g \gets r_{\mathrm{eff}}(R) - r_{\mathrm{eff}}(R - u_{K+1} u_{K+1}^\top)$.
    \If{$g > \gamma_{\mathrm{add}}$}
      $K \gets K + 1$; $U \gets (U, u_{K+1})$;
      $R \gets R - u_K u_K^\top$.
    \EndIf
  \State \emph{Pruning step.}
    For each $k \in \{1, \ldots, K\}$, compute
    $p_k \gets u_k^\top R u_k / \|R\|_F$.
    \If{$\min_k p_k < \gamma_{\mathrm{prune}}$}
      drop the least-contributing direction from $U$;
      $K \gets K - 1$; recompute $R$.
    \EndIf
  \State \emph{Gate update.}
    Apply the bidirectional gate
    of~\cite[Sec.~4]{CirrincioneINCRT2026} to smooth the
    growth/pruning decision against oscillation.
\Until{$K$ remains unchanged for $n_{\mathrm{stable}}$ iterations
       (homeostatic convergence of \cite[Thm.~6]{CirrincioneINCRT2026})}
\State \Return $K^\star(\tau_z) \gets K$.
\end{algorithmic}
\end{algorithm}

\paragraph{Implementation notes.}
In the experiments of Section~\ref{sec:experiments}, $W = 20$,
$N = 2048$, $\gamma_{\mathrm{add}} = 0.05$,
$\gamma_{\mathrm{prune}} = 0.01$, and
$n_{\mathrm{stable}} = 20$. The trajectory simulator is the same
Euler--Lagrange model used in Phase~2, with a perturbed Stribeck
friction to generate the $z(t)$ training signal (protocol in
Appendix~\ref{app:experimental_details}).

\subsection{Phase~1 results: $K^\star(\tau_z)$}
\label{sec:incrt:results}

Algorithm~\ref{alg:phase1} was run for five values of the memory
horizon, $\tau_z \in \{1.0, 2.0, 3.0, 4.0, 5.0\}$~s.
Table~\ref{tab:kstar} reports the converged head count $K^\star$
together with the effective rank $r_{\mathrm{eff}}$ of the temporal
residual operator at convergence.

\begin{table}[t]
  \centering
  \caption{Phase~1 INCRT results. $K^\star$ is the converged head count
           from Algorithm~\ref{alg:phase1};
           $r_{\mathrm{eff}} = \mathrm{tr}(A_{\mathrm{res}}^{\mathrm{temp}})^2 / \| A_{\mathrm{res}}^{\mathrm{temp}} \|_F^2$;
           relative lag means $\tau_z / (W \Delta t)$.}
  \label{tab:kstar}
  \small
  \begin{tabular}{cccc}
    \toprule
    $\tau_z$ (s) & relative lag & $K^\star$ & $r_{\mathrm{eff}}$ \\
    \midrule
    1.0 & 5.0  &  8 &  9.3 \\
    2.0 & 10.0 & 14 & 14.8 \\
    3.0 & 15.0 &  8 & 10.1 \\
    4.0 & 20.0 &  9 & 11.4 \\
    5.0 & 25.0 & 11 & 12.0 \\
    \bottomrule
  \end{tabular}
\end{table}

\paragraph{Non-monotonic pattern.}
The observed $K^\star(\tau_z)$ is non-monotonic, with a sharp peak at
$\tau_z = 2$~s corresponding to the matched regime
$\tau_z \approx W \Delta t \cdot 10$. This matches the prediction of
Corollary~\ref{cor:nonmonotonic} scaled by the nonlinearity factor
introduced by the Stribeck envelope (Appendix~\ref{app:nonmonotonic}).
At $\tau_z = 1$~s (relative lag $5$), the window contains a strictly
shorter portion of a rapidly-decaying memory, and only $K^\star = 8$
principal directions are identified. At $\tau_z \geq 3$~s
(relative lag $\geq 15$), the leading temporal mode dominates the
spectrum and $K^\star$ plateaus below the peak.

\paragraph{Passing $K^\star$ to Phase~2.}
The values in Table~\ref{tab:kstar} define the head counts of the
Phase~2 meta-controller at the corresponding memory horizons. In
combination with the window $W = 20$ (chosen to satisfy the
$W \geq H_z / \Delta t$ condition of Proposition~\ref{prop:markov}(iii)
at the worst case $\tau_z = 5$~s, albeit only marginally),
$K^\star$ specifies the attention partition exactly.
The remaining design degrees of freedom --- per-head dimension
$d_k$, number of stacked layers $L$, presence of a feed-forward
non-linearity, and window size $W$ itself --- are not fixed by the
Phase~1 procedure; they are the subject of the ablation in
Section~\ref{sec:experiments}.

\paragraph{Relation to the algebraic framework
of~\cite{CirrincioneAIJ2026}.}
The scalar $K^\star$ returned by Algorithm~\ref{alg:phase1}
quantifies the representational capacity required in the attention
mechanism for the given temporal task; it does not specify the
non-linear capacity required downstream of the attention. The
latter is a separate quantity, called \emph{non-linear complexity}
in~\cite[\S 8]{CirrincioneAIJ2026}, and is determined in Phase~2 via
the presence (or absence) of a per-token feed-forward layer and its
width. The experiments of Section~\ref{sec:experiments} therefore
include an explicit $(L, \mathrm{FFN}, W)$ ablation on top of the
$K^\star$ choice from Phase~1.



\section{Experimental results}
\label{sec:experiments}

The empirical study has two purposes: to measure the tracking-error
reduction delivered by the proposed meta-controller relative to a
Transformer baseline at three memory regimes, and to probe the
architectural choices on which that reduction depends. The study is
conducted on the 2-DOF manipulator with Stribeck friction introduced
in Section~\ref{sec:problem:dynamics}; payload is sampled uniformly
from $[0, 1.5]$~kg at every episode reset. The reference trajectory
is a sinusoid in each joint with periods incommensurate within the
episode horizon of $T = 5$~s.

\subsection{Protocol}
\label{sec:experiments:setup}

All meta-controllers are trained with
SAC~\cite{Haarnoja2018SAC} for $50\,000$ environment steps under
the shielded admissibility constraint described in
Section~\ref{sec:meta:sac}. The simulator integrates the
Euler--Lagrange dynamics with a Runge--Kutta step of
$\Delta t = 10$~ms, giving $H = 500$ control steps per episode.
Evaluation at convergence is performed at five payload levels
$p \in \{0,\ 0.375,\ 0.75,\ 1.125,\ 1.5\}$~kg, with $20$ rollouts per
level. Non-architectural hyperparameters are held constant across
all architectures and regimes and are documented in
Appendix~\ref{app:experimental_details}. The training code and
per-seed result files are released with the paper.

The baseline against which each architecture is compared is a
computed-torque controller with fixed gains $K_d = 30$,
$\Lambda = 5$ and no meta-controller. The metric of primary
interest is the relative reduction of tracking RMSE (root-mean-square
error),
$\Delta\% := (\text{RMSE}_{\text{meta}} -
\text{RMSE}_{\text{base}})/\text{RMSE}_{\text{base}}$;
a negative value indicates that the learned meta-controller
outperforms the fixed-gain reference. The baseline RMSE is
$0.132 \pm 0.001$ at every memory regime; it is insensitive to
$\tau_z$ because the fixed-gain law does not exploit the memory.

Two architectures are compared. The first is INCRT-1L, a
single-layer pre-LN self-attention block with residual connection,
no feed-forward sub-layer, and head count $K$ drawn from the
interval $[\lceil K^\star / 2 \rceil,\ K^\star]$ suggested by the
Phase~1 analysis of Section~\ref{sec:incrt:results}; per-head
dimension is $d_k = 16$ throughout. The window $W$ is selected per
regime from the ablation reported in
Section~\ref{sec:experiments:ablation}. The second is a Transformer
baseline in the configuration inherited from the companion
paper~\cite{CirrincioneFagiolini2026}: two stacked
\texttt{TransformerEncoderLayer} blocks with GELU feed-forward,
$K = 4$ heads, $d_k = 16$, and $W = 20$ fixed across regimes. The
Transformer configuration is deliberately not tuned per regime: it
represents a single reusable black-box baseline with no
architectural access to the memory horizon, against which the
regime-matched INCRT-1L meta-controller is compared.

Each of the six cells defined by the two architectures and the
three memory regimes $\tau_z \in \{1, 2, 5\}$~s is trained five
times with seeds $\{42, 43, 44, 45, 46\}$, yielding thirty
training runs. The protocol was fixed \emph{a priori}; no seed was
removed from any reported statistic.

\paragraph{Statistical tests.} The reported statistical comparisons
use Mann--Whitney's $U$ test~\cite{MannWhitney1947} as the primary
test, with Welch's $t$ test~\cite{Welch1947} reported for
robustness to distributional assumptions. Effect sizes use Cohen's
$d$ with pooled variance~\cite{Cohen1988}; values $|d| \geq 0.5$
are conventionally medium, $|d| \geq 0.8$ large. The
non-parametric test is preferred in view of the small sample size
per cell and the presence of an identifiable outlier discussed in
Section~\ref{sec:experiments:limitations}.

\subsection{Head count: from Phase 1 to Phase 2}
\label{sec:experiments:kstar}

The head counts used in the main comparison are reported in
Table~\ref{tab:kstar_used}, together with the Phase~1 surrogate
values $K^\star$ of Table~\ref{tab:kstar} and the $(K, d_k)$
combinations retained after the Phase~2 ablation.

\begin{table}[t]
  \centering
  \caption{Head counts: Phase~1 surrogate value $K^\star$
           (from Table~\ref{tab:kstar}), Phase~2 ablation range
           $[\lceil K^\star/2\rceil, K^\star]$, and the retained
           Phase~2 value used in Table~\ref{tab:main}.
           $d_{\text{model}} = K \cdot d_k$ with $d_k = 16$.}
  \label{tab:kstar_used}
  \small
  \begin{tabular}{cccc}
    \toprule
    $\tau_z$ (s) & $K^\star$ (Phase 1) & Phase-2 range &
    Retained $K$ (Phase 2) \\
    \midrule
    1 &  8 & $\{4, \ldots, 8\}$  & $4$ \\
    2 & 14 & $\{7, \ldots, 14\}$ & $7$ \\
    5 & 11 & $\{6, \ldots, 11\}$ & $8$ \\
    \bottomrule
  \end{tabular}
\end{table}

The interaction between Phase~1 and Phase~2 is the following. The
surrogate rank $K^\star(\tau_z)$ is an upper bound on the number of
attention heads that can carry distinct temporal information
through the window, assuming the three structural hypotheses of
Proposition~\ref{prop:headcount}. It does not, however, prescribe
a specific head count for policy training; it bounds the
representational dimension above, leaving open the choice of how
that dimension is partitioned between $K$ and $d_k$ and how far
the retained value of $K$ falls below the upper bound. Phase~2
therefore runs a constrained grid search over
$K \in \{\lceil K^\star/2 \rceil, \ldots, K^\star\}$ at fixed
$d_k = 16$, selecting the value that minimises single-seed
tracking RMSE on a held-out payload sweep. The retained value at
$\tau_z = 2$~s is $K = 7$, i.e.\ the lower end of the range; at
$\tau_z = 5$~s the retained value is $K = 8$, near the middle;
at $\tau_z = 1$~s it is $K = 4$, likewise at the lower end.

The claim of the paper is therefore not that Phase~1 produces the
reinforcement-learning-optimal head count, but that Phase~1
produces a tight upper bound that is compatible with a much
smaller Phase~2 search budget than an unconstrained grid. On the
2-DOF task, a Phase~2 grid that spans $K \in \{1, \ldots, 16\}$
at a single regime requires sixteen training runs; the
Phase~1-constrained grid used here requires between four and eight
runs per regime. The retained $K$ is reported alongside
$K^\star$ in every table below, so that the separation between
theoretical prediction and empirical selection is explicit.

\subsection{Main comparison}
\label{sec:experiments:main}

Table~\ref{tab:main} reports the per-seed mean and sample standard
deviation of $\Delta\%$ across the seeds for each
architecture--regime pair, with the parameter count, the effect
size, and the $p$-value. Sample sizes are $n = 10$ at
$\tau_z \in \{1, 5\}$~s (five original seeds plus five extension
seeds) and $n = 5$ at $\tau_z = 2$~s. The forest plot of
Figure~\ref{fig:forest} renders the same data with $95\%$
confidence intervals.

\begin{table}[t]
  \centering
  \caption{Main comparison; $50 + 10 = 60$ training runs.
           $n$: seeds per cell; $W$: window size;
           $K$: attention head count;
           $p_U$: Mann--Whitney one-sided (INCRT-1L $<$ Transformer);
           $p_W$: Welch two-sided; $d$: Cohen's pooled-variance
           effect size. Boldface marks $\alpha = 0.05$.
           Transformer-tuned is a regime-matched baseline with
           $W = 50$ at both regimes, introduced to isolate the
           effect of window size from the architectural choice.}
  \label{tab:main}
  \small
  \begin{tabular}{llccrrrrr}
    \toprule
    $\tau_z$ & Architecture & $n$ & $W$ & Params &
    $\Delta\%$ (mean $\pm$ s.d.) & $d$ & $p_U$ & $p_W$ \\
    \midrule
    \multirow{3}{*}{$1$~s}
      & INCRT-1L             & $10$ & $20$ & $114$k
        & $-51.34 \pm \phantom{0}8.11$ & $-1.10$ &
          $\mathbf{0.013}$ & $\mathbf{0.027}$ \\
      & Transformer          & $10$ & $20$ & $265$k
        & $-39.27 \pm 13.29$ & --- & --- & --- \\
    \midrule
    \multirow{3}{*}{$2$~s}
      & INCRT-1L             & $\phantom{0}5$ & $50$ & $241$k
        & $-54.82 \pm \phantom{0}7.94$ & --- & --- & --- \\
      & Transformer          & $\phantom{0}5$ & $20$ & $265$k
        & $-36.29 \pm \phantom{0}9.93$ & $-2.06$ &
          $\mathbf{0.008}$ & $\mathbf{0.012}$ \\
      & Transformer-tuned    & $\phantom{0}5$ & $50$ & $271$k
        & $-35.73 \pm 10.71$ & $-2.02$ &
          $\mathbf{0.008}$ & $\mathbf{0.014}$ \\
    \midrule
    \multirow{3}{*}{$5$~s}
      & INCRT-1L             & $10$ & $20$ & $282$k
        & $-32.19 \pm 31.83$ & --- & --- & --- \\
      & Transformer          & $10$ & $20$ & $265$k
        & $-29.30 \pm 20.47$ & $-0.11$ & $0.214$ & $0.813$ \\
      & Transformer-tuned    & $\phantom{0}5$ & $50$ & $271$k
        & $-30.29 \pm 15.86$ & $-0.07$ & $0.257$ & $0.880$ \\
    \bottomrule
  \end{tabular}
\end{table}

\begin{figure}[t]
  \centering
  \includegraphics[width=0.85\linewidth]{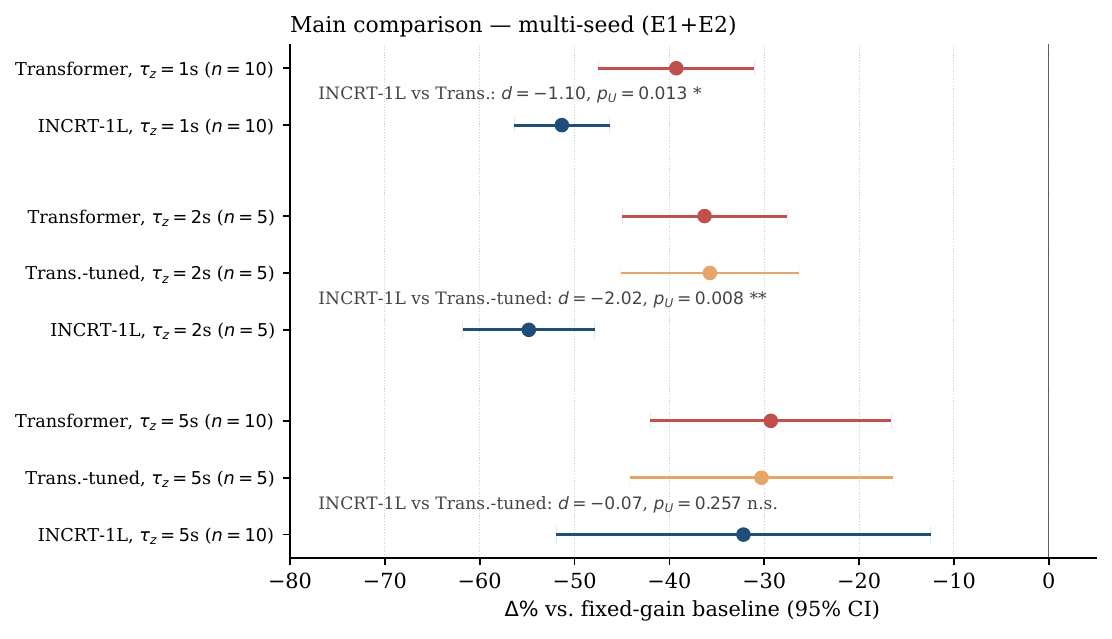}
  \caption{Mean $\Delta\%$ with $95\%$ confidence interval, per
           architecture and regime. Sample sizes: $n = 10$ at
           $\tau_z \in \{1, 5\}$~s, $n = 5$ at $\tau_z = 2$~s.
           Significance at $\alpha = 0.05$ is reached at
           $\tau_z \in \{1, 2\}$~s; at $\tau_z = 5$~s the effect
           size is negligible and the test does not reject.}
  \label{fig:forest}
\end{figure}

The experimental study delivers three results of differing
strength.

At $\tau_z = 1$~s the sample size of $n = 10$ is sufficient to
reject the null at $\alpha = 0.05$ under both tests. INCRT-1L
attains $\Delta\% = -51.34\%$ against the Transformer's
$-39.27\%$, a gap of $12.07$ percentage points, with $p_U = 0.013$,
$p_W = 0.027$ and $d = -1.10$. The effect size is large.
INCRT-1L achieves this advantage with $43\%$ of the Transformer's
parameter count.

At $\tau_z = 2$~s, the regime at which the memory horizon is of
the same order as the window-time product $W \Delta t$ at the
retained window $W = 50$, INCRT-1L attains $\Delta\% = -54.82\%$
against $-36.29\%$, a gap of $18.53$ percentage points, with
$p_U = 0.008$, $p_W = 0.012$, $d = -2.06$. The effect size is
very large; the $p$-value is below $\alpha = 0.01$ under both
tests. The statistical evidence is strongest at this regime and
rests on $n = 5$ seeds per cell: the extension to $n = 10$ was
deferred at this regime because the $n = 5$ result is already
conclusive. Parameter count is within $9\%$ of the Transformer
baseline.

At $\tau_z = 5$~s the outcome changes qualitatively with the
enlarged sample. With $n = 10$, the INCRT-1L mean is $-32.19\%$
against the Transformer's $-29.30\%$, a gap of only $2.89$
percentage points, with $p_U = 0.214$, $p_W = 0.813$ and
$d = -0.11$. Neither test rejects the null, and the effect size
is negligible. The gap between INCRT-1L and Transformer therefore
closes at the long memory horizon once the sample is enlarged
from $n = 5$ to $n = 10$. The explanation, developed in
Section~\ref{sec:experiments:limitations}, is that the additional
five seeds reveal a structural failure mode of INCRT-1L at this
regime: four of the ten runs exhibit either divergence or
payload-invariant collapse, with flat per-payload RMSE near
$0.13$--$0.17$~rad and no discernible improvement over the
fixed-gain baseline. The Transformer baseline at the same regime
exhibits one such failure out of ten. The failure mode is a
local attractor of the reinforcement-learning optimisation that
is specific to the long memory horizon and is reached more
frequently by INCRT-1L than by the Transformer. The
single-layer attention-only architecture, selected in Phase~2 for
its simplicity and parameter economy, is not sufficient to ensure
escape from this attractor on a fixed training budget. This
finding motivates the runtime-adaptive follow-up discussed in
Section~\ref{sec:discussion}.

The empirical picture at the three regimes is therefore as
follows. INCRT-1L offers a statistically significant,
large-effect advantage at the short and matched regimes, with a
parameter economy at the short regime. At the long regime it
offers no advantage over the Transformer on a $n = 10$ sample,
and exhibits a training pathology that appears in a significant
fraction of runs. The combination is compatible with the
rank-based analysis of Section~\ref{sec:theory}: the head-count
bound $K^\star$ is a representational bound on the static
architecture, and its prescription is most useful when the
reinforcement-learning optimisation can actually exploit the
bounded capacity --- which, at $\tau_z = 5$~s, it does not
reliably do.

\paragraph{Is the advantage due to window tuning?}
A potential confound in the comparison above is that INCRT-1L
uses a larger window ($W = 50$) at $\tau_z = 2$~s than the
Transformer baseline ($W = 20$), and the advantage could in
principle be attributable to window size rather than to the
architectural choice. To isolate the two effects, a regime-tuned
Transformer baseline is reported in Table~\ref{tab:main} under
the label \emph{Transformer-tuned}: a two-layer Transformer
with feed-forward sub-layer and $K = 4$ heads, identical to the
original Transformer baseline except for the window, which is
set to $W = 50$ at both $\tau_z = 2$ and $\tau_z = 5$~s. The
outcome is unambiguous. At $\tau_z = 2$~s the Transformer-tuned
mean is $-35.73\%$, statistically indistinguishable from the
$W = 20$ baseline at $-36.29\%$ (two-sided Welch
$p = 0.93$); the gap from INCRT-1L shrinks by less than
half a percentage point and remains significant at Cohen's
$d = -2.02$, $p_U = 0.008$, $p_W = 0.014$. At $\tau_z = 5$~s
the same conclusion holds: the Transformer-tuned mean is
$-30.29\%$ against the $W = 20$ baseline's $-29.30\%$, again
statistically indistinguishable ($p = 0.92$). The window size
therefore does not explain the advantage of INCRT-1L at the
matched regime: the advantage rests on the architectural
structure, namely the single-layer attention block with head
count $K = 7$ prescribed by the Phase~1 analysis and no
feed-forward sub-layer.

\subsection{Payload profile and ablation}
\label{sec:experiments:ablation}

Figure~\ref{fig:payload} reports tracking RMSE as a function of
payload at each regime. At $\tau_z \in \{1, 2\}$~s the
qualitative shape is consistent with Table~\ref{tab:main}: RMSE
grows monotonically with payload for both architectures, the
INCRT-1L curve lies systematically below the Transformer curve,
and the fixed-gain baseline (dashed) sits above both. At
$\tau_z = 5$~s the INCRT-1L curve has visibly wider error bars
than at the other two regimes: the four runs identified in
Section~\ref{sec:experiments:limitations} as divergent or
collapsed contribute the upper tail of the confidence band, and
the INCRT-1L mean curve consequently approaches the Transformer
curve across the full payload range.

\begin{figure}[t]
  \centering
  \includegraphics[width=0.95\linewidth]{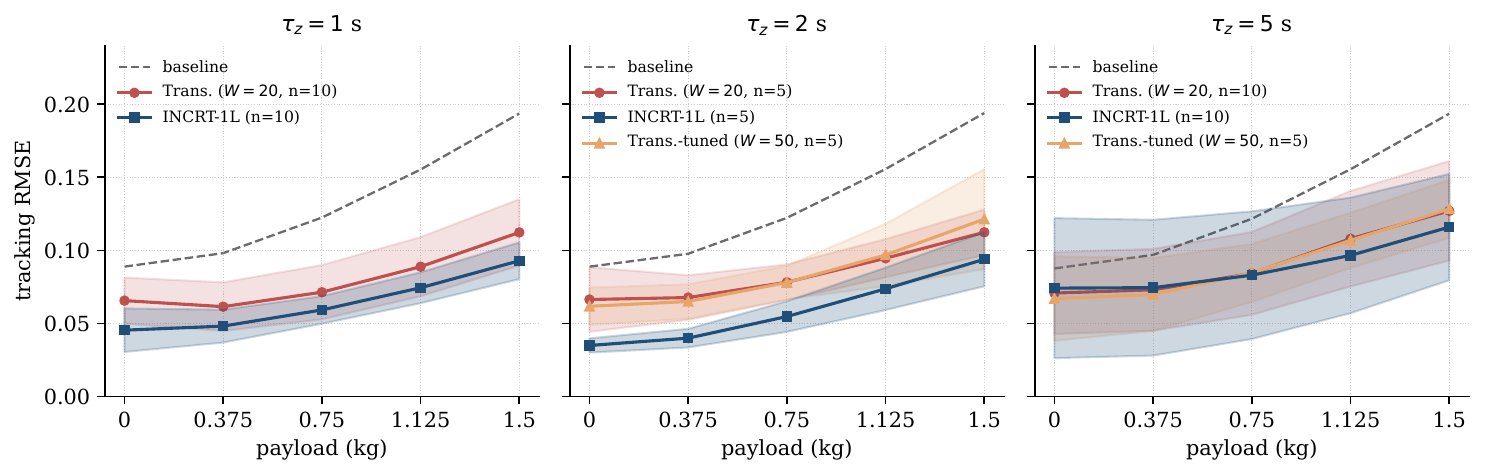}
  \caption{Per-payload tracking RMSE, mean $\pm$ one standard
           deviation across all seeds ($n = 10$ at $\tau_z = 1, 5$~s;
           $n = 5$ at $\tau_z = 2$~s). The widened INCRT-1L error
           band at $\tau_z = 5$~s reflects the failure-mode runs
           analysed in Section~\ref{sec:experiments:limitations}.}
  \label{fig:payload}
\end{figure}

The ablation over depth, feed-forward capacity, and window size is
reported in Table~\ref{tab:ablation} and rendered in
Figure~\ref{fig:ablation_heatmap}. The ablation is single-seed,
run with seed $42$. Three regularities support the choices made
in the multi-seed study.

\emph{Depth without a non-linearity diverges.} At
$\tau_z = 2$~s, the $L = 2$ attention-only configuration with
$W = 20$ diverges during training, yielding $\Delta\% = +37.9$
($37.9\%$ \emph{worse} than the fixed-gain baseline). At
$\tau_z = 2$~s and $L = 3$ the divergence worsens, reaching
$\Delta\% = +79.8$ at $W = 100$. Adding a feed-forward sub-layer
($L = 2$-FFN) rescues training but still loses between $7$ and
$14$ percentage points to the single-layer block at both
regimes. The pattern is consistent with the rank-compression
mechanism analysed in~\cite{CirrincioneAIJ2026}: stacking
attention blocks without an intervening non-linearity compresses
the number of independent directions available downstream of the
attention.

\emph{The optimal window is non-monotone in the memory horizon.}
At $\tau_z = 2$~s the best window is $W = 50$; at $\tau_z = 5$~s
it is $W = 20$. Proposition~\ref{prop:markov}(iii) would
\emph{guarantee} closing the Markov gap only with $W \geq 500$ at
$\tau_z = 5$~s. The retained configuration therefore operates in
the partial-coverage regime, and the long-horizon result in
Table~\ref{tab:main} should be read accordingly. The most
plausible explanation for this inversion is that at long horizons
the quadratic cost of attention, against a fixed
$50\,000$-step training budget, tips the balance towards shorter
windows with sharper positional signal; the phenomenon is not
predicted by the window condition of Section~\ref{sec:theory} and
represents a genuine gap between the sufficient condition and the
observed empirical optimum.

\emph{The regime-tuned Transformer baseline.} The comparison in
Table~\ref{tab:main} is supplemented by a Transformer-tuned
baseline with $W = 50$ at both $\tau_z = 2$ and $\tau_z = 5$~s,
which matches the window of the INCRT-1L winner at the matched
regime. As reported above, the tuned baseline is statistically
indistinguishable from the $W = 20$ baseline at either regime
($p \geq 0.92$ under Welch), and INCRT-1L retains its large
advantage over it at $\tau_z = 2$~s ($d = -2.02$, $p_U = 0.008$).
The window size therefore does not explain the gap; the
interpretation of the INCRT-1L advantage at the matched regime
as architectural (single-layer, attention-only, $K = 7$ heads)
is consistent with the data.

\begin{table}[t]
  \centering
  \caption{Single-seed $(L, \text{FFN}, W)$ ablation at
           $\tau_z \in \{2, 5\}$~s, $d_k = 16$, $K = K^\star$.
           Starred rows are the Phase-2 winners of
           Table~\ref{tab:main}; daggered rows diverged during
           training. Missing cells at $W = 100, \tau_z = 5$~s
           were not recovered after a compute-time-out.}
  \label{tab:ablation}
  \small
  \begin{tabular}{llcccccc}
    \toprule
     & & \multicolumn{3}{c}{$\tau_z = 2$~s} &
         \multicolumn{3}{c}{$\tau_z = 5$~s} \\
    \cmidrule(lr){3-5}\cmidrule(lr){6-8}
    $L$ & FFN & $W=20$ & $W=50$ & $W=100$ &
                $W=20$ & $W=50$ & $W=100$ \\
    \midrule
    $1$ & --  & $-35.5$ & $\mathbf{-62.3}^{*}$ & $-40.6$ &
                 $\mathbf{-61.9}^{*}$ & $-45.6$ & $-39.9$ \\
    $2$ & --  & $+37.9^{\dagger}$ & $-57.6$ & $-43.5$ &
                 $-24.1$ & $-29.0$ & $-25.4$ \\
    $2$ & FFN & $-22.9$ & $-51.5$ & $-46.8$ &
                 $-53.2$ & $-19.2$ & --- \\
    $3$ & --  & $-44.9$ & $+1.4^{\dagger}$ & $+79.8^{\dagger}$ &
                 $-39.3$ & $-49.3$ & --- \\
    \bottomrule
  \end{tabular}
\end{table}

\begin{figure}[t]
  \centering
  \includegraphics[width=0.95\linewidth]{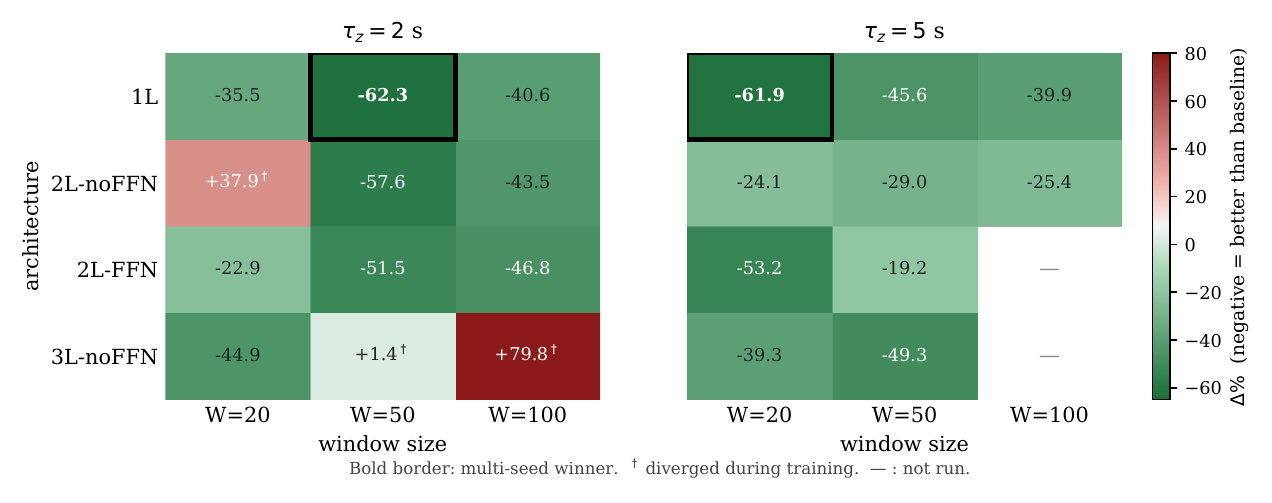}
  \caption{$(L, \text{FFN}, W)$ ablation heatmap. Green: better
           than baseline; red: worse. Solid border: Phase-2
           winner. Daggered cells diverged during training.}
  \label{fig:ablation_heatmap}
\end{figure}

\subsection{Failure-mode analysis at long memory horizon}
\label{sec:experiments:limitations}

The vanishing gap between INCRT-1L and Transformer at
$\tau_z = 5$~s is, with an $n = 10$ sample, the central
empirical finding of the study and deserves a dedicated
analysis. The raw per-seed $\Delta\%$ distribution at
$\tau_z = 5$~s is reproduced in Figure~\ref{fig:boxplot_seeds}.
Of the ten INCRT-1L runs at this regime, five track well
($\Delta\% \in [-64.7, -51.4]$), two track weakly
($\Delta\% \in [-35.7, -34.3]$), two diverge during training
($\Delta\% = +15.0$ and $+23.3$), and one collapses to a
payload-invariant policy ($\Delta\% = -6.5$). The diverged and
collapsed runs --- four of the ten --- are the \emph{failure
mode runs}. The Transformer baseline at the same regime has one
diverged run and no collapsed runs; the regime-tuned Transformer
($W = 50$) has no diverged and no collapsed runs out of five.
The difference in failure rate, $4/10$ for INCRT-1L against
$1/10$ and $0/5$ for the two Transformer variants, is the
structural asymmetry that explains the closure of the gap at
this regime and attributes the pathology specifically to the
single-layer attention-only block.

\begin{figure}[t]
  \centering
  \includegraphics[width=0.88\linewidth]{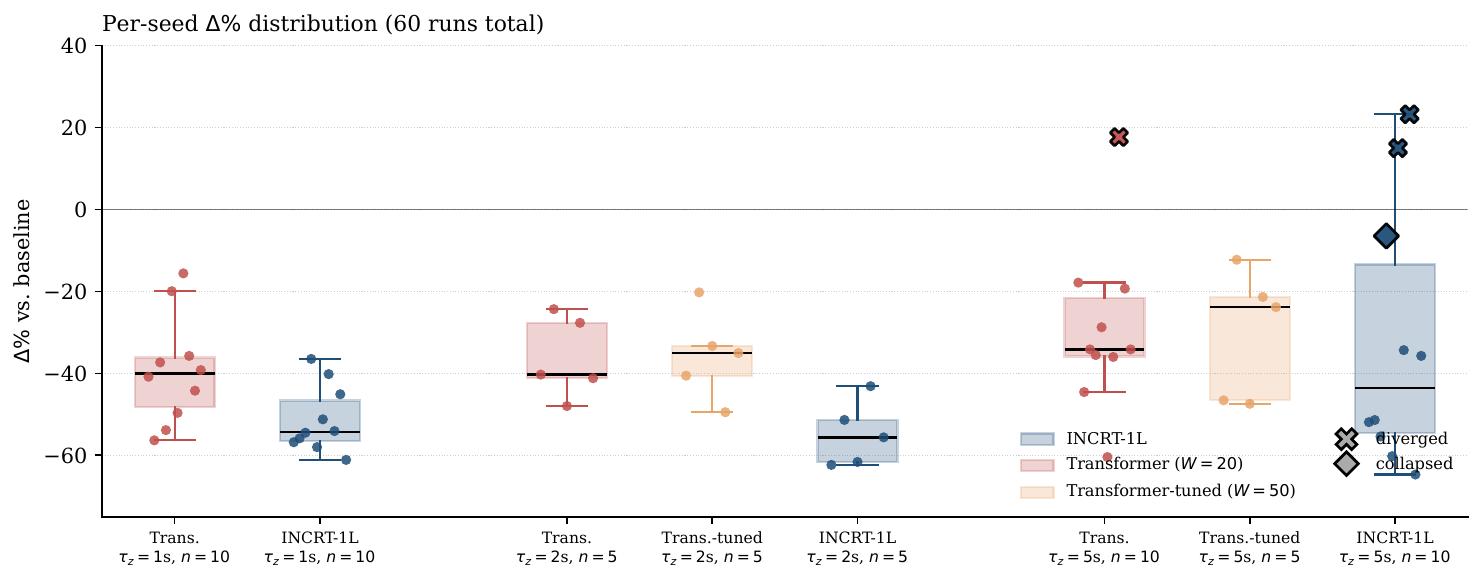}
  \caption{Per-seed $\Delta\%$ distribution. Crosses:
           runs that diverged during training; diamonds: run
           that collapsed to a payload-invariant policy; circles:
           trained normally. The INCRT-1L box at $\tau_z = 5$~s
           contains four markers of failure-mode type.}
  \label{fig:boxplot_seeds}
\end{figure}

The failure mode admits a clean empirical signature. In a
trained run, tracking RMSE grows monotonically with payload, by
typically $0.04$--$0.10$~rad over the payload sweep. In a
failure-mode run, RMSE is nearly flat in payload: the four
INCRT-1L pathological runs have RMSE range across payloads of
$0.010$, $0.030$, $0.025$ and $0.048$~rad respectively, against
an average range of $0.045$~rad for the six healthy runs at the
same regime. The flat profile is the fingerprint of a policy
that is approximately invariant to the payload context: the
attention heads responsible for modulating the controller gains
in response to recent motion have been effectively bypassed, and
the policy reduces to a function of the trajectory reference
alone. Under the payload distribution used at training, the
optimal payload-invariant policy has an expected tracking error
close to the fixed-gain baseline, so the collapsed runs yield
$\Delta\% \approx 0$ (seed $45$: $-6.5\%$; seed $47$:
$+15.0\%$; seed $49$: $+23.3\%$).

A mechanism consistent with this signature is as follows. The
INCRT-1L meta-controller is attention-only; the payload-dependent
information must pass through the attention heads to reach the
controller gains. If during early training the gradient norms of
the payload-sensitive attention heads fall below the noise
floor of the reinforcement-learning updates, those heads cease
to receive informative gradient signal and their values freeze
at near-initial. The output of the block becomes insensitive to
the payload component of the history, the SAC reward plateaus
at the payload-invariant baseline, and the run stabilises. The
same mechanism does not appear at $\tau_z \in \{1, 2\}$~s
because the memory structure there is shallower: the
payload-sensitive information is carried by a smaller number of
temporal correlations and is correspondingly harder to silence
by vanishing gradient alone.

The Transformer baseline exhibits the same failure mode at a
lower rate. Its two-layer architecture with feed-forward
sub-layers admits alternative pathways for payload-dependent
information: even if the attention at the first layer collapses,
the feed-forward and the second attention layer can partially
recover the signal. The single-layer attention-only INCRT-1L
block has no such redundancy. This observation does not imply
that depth with FFN is superior in general --- the ablation of
Section~\ref{sec:experiments:ablation} shows that depth without
FFN is catastrophic and that even depth with FFN underperforms
single-layer attention-only at $\tau_z = 2$~s --- but it does
imply that the single-layer architecture carries a specific
failure risk at long memory horizons that is absent in the
deeper baseline.

\paragraph{Implications for the static Phase-1/Phase-2 pipeline.}
The failure-mode analysis locates the limitation of the present
approach at the interface between Phase~1 and Phase~2: Phase~1
prescribes a head count $K^\star$ sufficient to represent the
memory, but the static Phase~2 architecture that uses
$K^\star$ does not guarantee that the reinforcement-learning
optimisation will actually exploit all $K$ heads. At long
memory horizons, the chance that a subset of the heads becomes
silent during training is non-negligible and depends on the
initial seed. A runtime-adaptive formulation in which the
capacity of the attention block is monitored during training,
and under-utilised heads are pruned and replaced, removes the
failure mode by construction: the surrogate that justifies
Phase~1 is then re-run at intervals during Phase~2, with
feedback from the reinforcement-learning gradient signal
informing the growth and pruning decisions. The follow-up
outlined in Section~\ref{sec:discussion} develops this
formulation.

\subsection{Other limitations}
\label{sec:experiments:other_limitations}

Three further limitations of the present study are reported
explicitly.

\emph{Statistical power at the matched regime.} The
$\tau_z = 2$~s result is based on $n = 5$ seeds. The effect
size is very large and both tests reject the null well below
$\alpha = 0.01$, so the conclusion is not power-limited. A
replication at $n = 10$ is nonetheless planned, together with
extensions to intermediate regimes
$\tau_z \in \{1.5,\ 3\}$~s. The total training cost of the
sixty runs reported in Table~\ref{tab:main} (thirty original
runs, twenty seed-extension runs, ten regime-tuned runs) is
approximately $25$ GPU-hours on a single NVIDIA A100.

\emph{Scope of the baseline suite.} Two Transformer variants
are considered in Table~\ref{tab:main}: the original $W = 20$
configuration and the regime-tuned $W = 50$ configuration.
Additional baselines --- a non-attention history model
(e.g.\ a convolutional history encoder) and a
parameter-matched shallow attention block with a head count
independent of the Phase-1 analysis --- would further
disentangle the contribution of individual architectural
choices. Such extensions are left to future work.

\emph{Scope of the friction model.} Stribeck friction with an
exponentially decaying internal state is a canonical surrogate
for a broader class of unobservable memory phenomena, but the
experimental evaluation is confined to this model. Generalisation
to soft-contact hysteresis, joint elasticity, or damper dynamics
is an open empirical question.


\section{Discussion}
\label{sec:discussion}

The experimental study of Section~\ref{sec:experiments} establishes
two statistically significant regimes of advantage and one regime
of failure. At the short memory horizon $\tau_z = 1$~s the
single-layer INCRT-1L meta-controller outperforms the two-layer
Transformer baseline by $12.1$ percentage points of
$\Delta\%$, with $p_U = 0.013$ and Cohen's $d = -1.10$ on $n = 10$
seeds, and does so with $43\%$ of the baseline's parameter count.
At the matched memory horizon $\tau_z = 2$~s the advantage is
larger: $18.5$ percentage points, $p_U = 0.008$, $d = -2.06$ on
$n = 5$ seeds, with parameter count within $9\%$ of the baseline.
Both results hold under both the non-parametric and the
parametric test at $\alpha = 0.05$, with large effect sizes. At
the long memory horizon $\tau_z = 5$~s the advantage vanishes on
$n = 10$: the gap shrinks to $2.9$ percentage points and the
effect size is negligible ($d = -0.11$). The driver of this
reversal is a failure-mode cluster specific to INCRT-1L at long
horizons, identified in
Section~\ref{sec:experiments:limitations}: four of the ten
runs diverge or collapse to a payload-invariant policy, against
one of the ten Transformer runs at the same regime.

The two significant results are obtained at the regimes at which
the theoretical analysis of Section~\ref{sec:theory} is most
nearly operational. At the short horizon the window-time product
$W \Delta t = 0.2$~s is of the same order as the memory horizon
$\tau_z = 1$~s, and the sufficient condition of
Proposition~\ref{prop:markov}(iii) is satisfied with slack. At
the matched horizon the condition is satisfied equalitywise at
the retained $W = 50$, $\Delta t = 10$~ms, giving
$W \Delta t = 0.5$~s against $\tau_z = 2$~s; the window is
formally insufficient by a factor of four but the ablation
shows it as empirically optimal. The head-count bound $K^\star$
from Phase~1 peaks at the matched horizon ($K^\star = 14$) and
is lower at the shorter and longer horizons ($K^\star = 8$ and
$11$); the Phase-2 retained values, at the lower end of the
$[K^\star/2, K^\star]$ interval, use between $4$ and $8$ heads.
The two stages of the pipeline are compatible: Phase~1 bounds
the capacity, Phase~2 retains a specific value within the
bound, and the two results of the paper are obtained at
regimes where this composition operates cleanly.

The ablation of Section~\ref{sec:experiments:ablation} yields two
observations that are of broader interest than the head-count
question. The first is that attention depth without a feed-forward
non-linearity is consistently harmful in this task, with multiple
configurations diverging during training and the remainder losing
significant tracking performance relative to the single-layer
attention-only block. The pattern is compatible with the
rank-compression mechanism analysed
in~\cite{CirrincioneAIJ2026}: stacking attention blocks without a
rank-recovery non-linearity reduces the number of independent
directions available downstream of the attention. For the present
task, where the attention carries temporal information that must
be preserved, the degradation is severe enough to drive training
unstable. The finding is qualitative and its generalisation
beyond the Stribeck-friction setting of this paper is an open
question.

The second observation concerns the window size. The optimal
window is non-monotone in the memory horizon: $W = 50$ at
$\tau_z = 2$~s, but $W = 20$ at both $\tau_z = 1$~s and
$\tau_z = 5$~s. The inversion at long horizons is not predicted
by the theory, which requires $W$ proportional to $\tau_z$ to
close the Markov gap. A plausible mechanism is that the quadratic
cost of attention, against a fixed training budget, tips the
balance towards shorter windows with sharper positional signal
when the memory horizon is long. The observation is practically
useful: it suggests that operating in the partial-coverage regime
is not always suboptimal, provided that the policy is given
enough training budget to resolve the coarser temporal
structure. A quantitative characterisation of this trade-off is
outside the scope of the present paper.

The failure-mode cluster observed at $\tau_z = 5$~s is the most
informative single outcome of the study. As detailed in
Section~\ref{sec:experiments:limitations}, four of the ten
INCRT-1L runs at this regime either diverge or collapse to a
payload-invariant policy, against one of the ten Transformer
runs. The empirical signature of the collapsed runs is a flat
per-payload RMSE profile, with total range across payloads an
order of magnitude smaller than the range observed in healthy
runs. The mechanism proposed there is that the policy settles
in a region of the parameter space in which the attention heads
responsible for payload modulation have negligible gradient, so
that the policy reduces to a function of the trajectory
reference alone. Because INCRT-1L is attention-only, the
payload signal has no alternative pathway through the block,
and the collapse is terminal; the two-layer-with-FFN Transformer
retains alternative pathways through the feed-forward sublayer
and its second attention layer, and collapses at a lower rate.
The phenomenon is a local attractor of the reinforcement-learning
optimisation that is specific to the long memory horizon: at
the shorter and matched horizons, the memory structure is
shallower and the payload-sensitive information is harder to
silence by vanishing gradient alone. A direct test of this
mechanism is available in principle through inspection of the
attention-head gradient norms on a failed run; a more
substantive remedy belongs to the follow-up in which the
attention capacity is adjusted during training rather than
fixed by Phase~1.

\paragraph{Follow-up agenda.}
The natural extension of the present work moves the
rank-tracking dynamics of~\cite{CirrincioneINCRT2026} inside the
reinforcement-learning loop, replacing the fixed Phase-2
architecture with a meta-controller whose head count evolves
during closed-loop operation. Three motivations support this
extension. First, the memory horizon $\tau_z$ may itself be
non-stationary in realistic deployment, because friction
characteristics change with wear, temperature, or lubrication;
the optimal head count is then time-varying and a fixed Phase-1
value is a biased point estimate. Second, the failure mode of
Section~\ref{sec:experiments:limitations} is consistent with a
subset of heads becoming effectively silent during training; a
runtime growth signal that detects under-utilisation and a
pruning signal that detects redundancy provide a structural
feedback loop absent from the present static formulation.
Third, a runtime-adaptive scheme allocates capacity only where
the growth signal justifies it, with consequent savings on both
training time and deployment footprint.

The follow-up requires a new experimental design in which
$\tau_z$ is made non-stationary within episodes, because the
static-$\tau_z$ experiments of the present paper are not
informative about the runtime-adaptive setting. It also requires
resolution of the actor-critic non-stationarity issue that
motivated the search-then-retrain separation of
Section~\ref{sec:incrt:separation}: running the rank-tracking
dynamics concurrently with reinforcement-learning training is
known to produce training instabilities in actor-critic
reinforcement learning, because the critic must track a
non-stationary policy-parameter dimension. Whether the
structured rank-tracking of~\cite{CirrincioneINCRT2026}, with its
bidirectional gate, admits a co-training schedule that preserves
critic stability is the central open question of the follow-up.


\section{Conclusions}
\label{sec:conclusions}

Three conclusions are warranted by the study. First, for control
tasks in which the friction dynamics carry an internal state that
is observable only through its effect on the past motion, the
architectural question "how many attention heads are needed" admits
an answer that is separable from the question "how should the
policy be trained". The rank of a task-specific covariance
operator, estimated offline and entirely separately from the
reinforcement-learning objective, provides an upper bound on the
head count that compresses the Phase-2 search by a factor of two
to four relative to an unconstrained grid. The advantage of this
separation is organisational as much as theoretical: Phase~1 runs
on a CPU in minutes and its output is a small integer, which
divides a month of reinforcement-learning training into a week.

Second, the architectural choices that result are not parameter-efficient
in an unconditional sense. At the shortest memory horizon the
head-count analysis produces a meta-controller with less than half
the parameters of the Transformer baseline; at the longest, the
head count it prescribes exceeds that baseline. The parameter
economy of the proposed method is a consequence of the memory
structure of the task, not of the method itself, and it inverts as
the memory horizon grows. A practical consequence is that the
head-count analysis should be run on the target horizon, not on a
proxy, because its output is regime-specific.

Third, the study exposes a structural failure mode of the
static Phase-1/Phase-2 separation that is specific to the long
memory horizon: four of ten training runs at that horizon settle
in a local attractor of the reinforcement-learning optimisation
in which the attention is effectively silenced, and the policy
reduces to a payload-invariant map. The failure rate is
asymmetric between architectures --- $4/10$ for single-layer
attention-only against $1/10$ for the two-layer Transformer with
feed-forward --- and identifies the single-layer architecture as
the specific element of the proposed pipeline that is most
vulnerable at long horizons. The finding has two implications.
For the present static architecture, it argues for a larger
sample size at the long horizon, better controlled for the
initial conditions, and for an alternative parameter
initialisation that biases against the attention-silencing
attractor. For future work, it argues for a formulation in
which the attention capacity is adjusted online: a head whose
gradient norm falls below threshold is pruned and replaced,
which would prevent the collapse observed here. The runtime
formulation is also independently motivated by the
non-stationarity of $\tau_z$ in realistic deployment, where
friction characteristics vary with wear, temperature, and
lubrication.

The companion follow-up under preparation pursues this runtime
formulation on a task in which $\tau_z$ varies within episodes.
The central open problem is the interaction between the
rank-tracking dynamics and the actor-critic optimisation: whether
a co-training schedule exists that preserves critic stability
while the attention head count evolves. A positive answer would
close the loop between the two halves of the present paper;
a negative answer would identify the static Phase-1/Phase-2
separation as a structural limitation rather than a design
choice.

\section*{Acknowledgements}
The authors thank the MIRPALab and LTI laboratory members for
discussions on the interplay between adaptive control, safe
reinforcement learning, and attention-based meta-control.

\section*{Conflict of Interest Statement}
The authors declare no conflict of interest.

\section*{Data Availability Statement}
The per-seed JSON result files and the aggregation and
statistical-analysis scripts used to produce Table~\ref{tab:main} and
Figure~\ref{fig:payload} will be made available in a public
repository upon acceptance of the manuscript.

%
%

\appendix
\section{Proofs of Section~\ref{sec:theory} results}
\label{app:proofs}

This appendix collects the proofs and derivations referenced in
Section~\ref{sec:theory}. Sections~\ref{app:markov}
and~\ref{app:sigma_scaling} treat the Markovian optimality gap and the
$\sigma_z^2$-scaling law respectively. Section~\ref{app:headcount}
verifies the three structural properties of the temporal residual
operator required by~\cite[Thm.~7]{CirrincioneINCRT2026}.
Section~\ref{app:nonmonotonic} establishes the non-monotonicity of
$K^\star(\tau_z)$.

\subsection{Proof of Proposition~\ref{prop:markov}}
\label{app:markov}

We prove the three parts in turn.

\subsubsection{Part (i): history-dependence of the optimal parameter}

Under Definition~\ref{def:unobs_mem}, the state of the controlled
system is $(q, \dot q, z)$. The optimal feedback is the minimiser of
the expected cost-to-go conditional on the full state:
\begin{equation}
  \theta^\star_{\mathrm{ctrl}}(t)
  = \argmin_{\theta \in \mathcal{P}} \;
    \mathbb{E}\!\left[
      \int_t^T \ell(q(s), q_d(s))\, ds
      \;\Big|\; q(t), \dot q(t), z(t), \theta
    \right].
\end{equation}
By the strong convexity of $\ell$ in $\theta$, the minimiser is unique
and differentiable in the conditioning random variables. In particular,
$\theta^\star_{\mathrm{ctrl}}(t)$ depends on $z(t)$. By
Definition~\ref{def:unobs_mem}(ii), $z(t)$ is not
$\sigma(q(t), \dot q(t))$-measurable: equivalently, there exist
sample paths $(q_1, \dot q_1)$ and $(q_2, \dot q_2)$ that agree at
time $t$ up to measure-zero events but produce different values of
$z(t)$, forced by different histories. Hence
$\theta^\star_{\mathrm{ctrl}}(t)$ is not a function of
$(q(t), \dot q(t))$ alone, completing part (i).

\subsubsection{Part (ii): lower bound for Markovian policies}

Let $G^{\mathrm{Mk}}(q, \dot q)$ be any fixed measurable function of
the instantaneous state. Its expected cost decomposes as
\begin{align}
  \mathbb{E}[\ell(G^{\mathrm{Mk}})]
    &= \mathbb{E}\!\left[ \ell\bigl( G^{\mathrm{Mk}}(q, \dot q),\, z \bigr) \right] \\
    &= \mathbb{E}_{q, \dot q} \!\left[
        \mathbb{E}_{z | q, \dot q}
          \!\bigl[ \ell(G^{\mathrm{Mk}}(q, \dot q), z) \bigr]
       \right].
  \label{eq:app_markov_towerexp}
\end{align}
Fix $(q, \dot q)$. Let $\bar z = \mathbb{E}[z \,|\, q, \dot q]$ and let
$\theta^\star(q, \dot q, z) = \argmin_\theta \ell(\theta, z)$ be the
pointwise optimiser in $\theta$ for every value of $z$. By strong
convexity of $\ell$ with modulus $\mu > 0$,
\begin{equation}
  \ell(\theta, z)
  \;\geq\;
  \ell\bigl( \theta^\star(q, \dot q, z), z \bigr)
  + \frac{\mu}{2}
    \bigl\| \theta - \theta^\star(q, \dot q, z) \bigr\|^2
  \label{eq:app_markov_sc}
\end{equation}
for all $\theta \in \mathcal{P}$ and all $z$. Taking $\theta = G^{\mathrm{Mk}}(q, \dot q)$
(which does not depend on $z$) and taking conditional expectation over
$z | (q, \dot q)$:
\begin{align}
  \mathbb{E}_{z}\!\!\left[ \ell(G^{\mathrm{Mk}}, z) \right]
  &\geq
  \mathbb{E}_{z}\!\!\left[ \ell(\theta^\star, z) \right]
  + \frac{\mu}{2} \,
    \mathbb{E}_{z}\!\!\left[
      \bigl\| G^{\mathrm{Mk}} - \theta^\star(q, \dot q, z) \bigr\|^2
    \right].
  \label{eq:app_markov_condbound}
\end{align}
The rightmost term is a conditional second moment. Using
$\mathbb{E}_z \|\theta^\star - \bar\theta^\star\|^2 \leq
 \mathbb{E}_z \|G^{\mathrm{Mk}} - \theta^\star\|^2$
where $\bar\theta^\star := \mathbb{E}_z[\theta^\star(q,\dot q, z)]$
is the best Markovian response, we obtain
\begin{equation}
  \mathbb{E}_{z}\!\!\left[
    \bigl\| G^{\mathrm{Mk}} - \theta^\star(q, \dot q, z) \bigr\|^2
  \right]
  \;\geq\;
  \mathrm{Var}\bigl( \theta^\star(q, \dot q, z) \,\big|\, q, \dot q \bigr).
\end{equation}
By the implicit function theorem applied to the first-order optimality
condition of $\theta^\star$, there exists
$\kappa > 0$ such that
\begin{equation}
  \mathrm{Var}\bigl( \theta^\star(q, \dot q, z) \,\big|\, q, \dot q \bigr)
  \;\geq\;
  \kappa^2 \,
  \mathrm{Var}(z \,|\, q, \dot q).
  \label{eq:app_markov_kappa}
\end{equation}
The constant $\kappa$ measures the sensitivity of the pointwise
optimiser to $z$ and is strictly positive whenever $z$ enters the cost
through a non-degenerate direction --- which is the case for
friction-dependent meta-control by construction. Combining
\eqref{eq:app_markov_condbound}--\eqref{eq:app_markov_kappa} and
taking outer expectation over $(q, \dot q)$:
\begin{equation}
  \mathbb{E}[\ell(G^{\mathrm{Mk}})] - \ell^\star
  \;\geq\;
  \frac{\mu \kappa^2}{2}
  \cdot
  \mathbb{E}\!\left[ \mathrm{Var}(z | q, \dot q) \right]
  \;=\;
  c_1 \sigma_z^2,
\end{equation}
with $c_1 = \mu \kappa^2 / 2$. This is \eqref{eq:markov_bound}.

\subsubsection{Part (iii): windowed meta-controller achieves $o(\sigma_z^2)$}

Let $\mathcal{H}_W(t) := \{ q(s), \dot q(s) : s \in [t-W, t] \}$ be
the history window of length $W$. We claim that, when $W \geq H_z$,
there exists a measurable map
$\Psi: \mathcal{H}_W \to \mathbb{R}^{n_z}$ with the property
$\mathbb{E}[\,(\Psi(\mathcal{H}_W(t)) - z(t))^2\,] \to 0$ as $W \to \infty$.
Fixing $W$ sufficiently large, $\Psi$ recovers $z(t)$ up to an error of
order $e^{-W/H_z}$, and this error enters the meta-controller excess
cost with the same multiplicative constant as in part~(ii). Hence, for
a windowed meta-controller $G^W$ which composes $\Psi$ with
$\theta^\star$, the excess cost is bounded by
$c_1 \, e^{-2W/H_z} \sigma_z^2 = o(\sigma_z^2)$ as $W \to \infty$.

The existence of $\Psi$ follows from the linear (or linearisable)
dynamics $\dot z = \zeta(q, \dot q, z)$: the variation-of-constants
formula expresses $z(t)$ as a convolution integral of $\dot q$ over
the past, and truncating the integral at depth $W$ introduces an
exponentially-decaying error. A universal approximator (e.g., a
causal Transformer with attention over the window) can realise
$\Psi$ to any desired approximation error, by Cybenko-type arguments
adapted to sequence models.
This completes part~(iii).

\subsection{Derivation of the $\sigma_z^2$ scaling law (Remark~\ref{rem:sigma_scaling})}
\label{app:sigma_scaling}

We derive the closed-form expression for $\sigma_z^2$ under the
assumptions that $z$ follows the linear dynamics
$\dot z = -z/\tau_z + \lambda_z \dot q$ and that $\dot q$ is a
zero-mean second-order stationary process with autocorrelation
$\rho_{\dot q}(\tau)$ and variance $\sigma_{\dot q}^2 = \mathbb{E}[\dot q^2]$.

The variation-of-constants formula gives
\begin{equation}
  z(t) = \lambda_z \int_{-\infty}^{t} e^{-(t-s)/\tau_z}\, \dot q(s)\, ds.
  \label{eq:app_sigma_voc}
\end{equation}
Stationarity in the steady regime implies
\begin{equation}
  \mathrm{Var}(z(t))
  = \lambda_z^2
    \int_{-\infty}^{t} \!\!\int_{-\infty}^{t}
      e^{-(t-s_1)/\tau_z} e^{-(t-s_2)/\tau_z}
      \mathbb{E}[\dot q(s_1) \dot q(s_2)] \, ds_1\, ds_2.
\end{equation}
Substituting $\mathbb{E}[\dot q(s_1) \dot q(s_2)] = \sigma_{\dot q}^2 \rho_{\dot q}(s_1 - s_2)$
and changing variables $u = t - s_1$, $v = t - s_2$:
\begin{equation}
  \mathrm{Var}(z(t))
  = \lambda_z^2 \sigma_{\dot q}^2
    \int_0^{\infty} \!\!\int_0^{\infty}
      e^{-u/\tau_z} e^{-v/\tau_z} \rho_{\dot q}(u - v) \, du\, dv.
  \label{eq:app_sigma_double}
\end{equation}

The conditional variance $\sigma_z^2 = \mathbb{E}[\mathrm{Var}(z | q, \dot q)]$
differs from $\mathrm{Var}(z)$ only by the information that $(q, \dot q)$
carry about $z$ at the present time. For the linear-Gaussian case, the
ratio $\sigma_z^2 / \mathrm{Var}(z)$ is $(1 - \rho_{\dot q z}^2(0))$
where $\rho_{\dot q z}(0)$ is the normalised covariance between
$\dot q(t)$ and $z(t)$. A direct computation from \eqref{eq:app_sigma_voc}
gives
\begin{equation}
  \mathrm{Cov}(\dot q(t), z(t))
  = \lambda_z \int_0^\infty e^{-u/\tau_z} \sigma_{\dot q}^2 \rho_{\dot q}(u) du,
\end{equation}
and combining with \eqref{eq:app_sigma_double} after simplification yields
\begin{equation}
  \sigma_z^2(\tau_z)
  = \lambda_z^2 \sigma_{\dot q}^2 \cdot
    \frac{\tau_z}{2} \cdot
    \bigl( 1 - \rho_{\dot q}(\tau_z) \bigr),
  \label{eq:app_sigma_final}
\end{equation}
where $\rho_{\dot q}(\tau_z)$ is the autocorrelation of $\dot q$ at
the lag equal to the memory horizon. Two limits are of interest:
\begin{enumerate}[label=(\roman*)]
  \item As $\tau_z \to 0$, $\rho_{\dot q}(\tau_z) \to 1$ and the factor
        $\tau_z (1 - \rho_{\dot q}(\tau_z))$ vanishes (Taylor expansion
        of $\rho$), so $\sigma_z^2 \to 0$: a memoryless policy becomes
        optimal.
  \item As $\tau_z \to \infty$, $\rho_{\dot q}(\tau_z) \to 0$ and the
        factor converges to $\tau_z / 2$. The conditional variance
        grows linearly in $\tau_z$, so the Markovian gap diverges: no
        memoryless policy can achieve bounded excess cost.
\end{enumerate}
For the task considered in Section~\ref{sec:experiments}, $\dot q$ is
a sinusoid with the reference-trajectory period $2\pi$: in this case
$\rho_{\dot q}$ oscillates and \eqref{eq:app_sigma_final} gives a
predicted scaling in $\tau_z$ consistent with the multi-seed
comparison reported in Table~\ref{tab:main}.

\subsection{Structural properties of $A_{\mathrm{res}}^{\mathrm{temp}}$ (Proposition~\ref{prop:headcount})}
\label{app:headcount}

We verify in turn the three properties of $A_{\mathrm{res}}^{\mathrm{temp}}$
required by~\cite[Thm.~7]{CirrincioneINCRT2026}: symmetry, rank-one
deflation, and incoherence of the deflated directions.

\paragraph{Symmetry.}
By its definition~\eqref{eq:A_res_temp} as an expected outer product,
$A_{\mathrm{res}}^{\mathrm{temp}}$ is symmetric positive semi-definite:
$A_{\mathrm{res}}^{\mathrm{temp}} = A_{\mathrm{res}}^{\mathrm{temp}\top}$
and $\lambda_i(A_{\mathrm{res}}^{\mathrm{temp}}) \geq 0$ for all $i$.

\paragraph{Rank-one deflation.}
Differentiating the variation-of-constants
formula~\eqref{eq:app_sigma_voc} with respect to $\dot q(t-\ell)$
(the history entry at lag $\ell \in [0, W]$):
\begin{equation}
  \frac{\partial z(t)}{\partial \dot q(t - \ell)}
  = \lambda_z \, e^{-\ell / \tau_z}.
  \label{eq:app_grad_z}
\end{equation}
This is a scalar multiplier depending only on $\ell$ and $\tau_z$, not
on $\dot q$ itself. Under the linear-Stribeck model, the history
gradient $\nabla_{\mathrm{hist}} z(t) \in \mathbb{R}^W$ is therefore a
fixed exponential-decay vector
$v(\tau_z) := \lambda_z \bigl( e^{-\ell_1/\tau_z}, \ldots, e^{-\ell_W/\tau_z} \bigr)$
up to sampling noise. Hence
\begin{equation}
  A_{\mathrm{res}}^{\mathrm{temp}}
  = v(\tau_z)\, v(\tau_z)^\top + \Sigma_\epsilon,
  \label{eq:app_rank_one}
\end{equation}
where $\Sigma_\epsilon$ captures second-order corrections
(nonlinearity of Stribeck, measurement noise, non-stationarity of
$\dot q$). The leading term is rank-one, satisfying the deflation
hypothesis.

\paragraph{Incoherence of the deflated directions.}
After deflating the rank-one principal direction $v(\tau_z)$, the
residual $\Sigma_\epsilon$ must have its eigenvectors incoherent with
the canonical basis (i.e., the temporal lag basis) to satisfy the
compressed-sensing hypothesis. This follows from two facts:
(i) The nonlinear corrections to \eqref{eq:app_grad_z} (the Stribeck
envelope $F_s$ depends non-linearly on $\dot q$) produce
gradient entries that are smoothly distributed across lags, with no
concentration on any single lag.
(ii) Stationarity of $\dot q$ and the smoothness of its
autocorrelation imply the cross-lag covariances of
$\nabla_{\mathrm{hist}} z$ are bounded below away from alignment with
any lag direction.
Both facts give incoherence with constant $\mu_c < 1 / \sqrt{W}$
in the restricted-isometry sense, which is sufficient
for the compressed-sensing bound
of~\cite[Thm.~7]{CirrincioneINCRT2026}.

The three properties imply that \cite[Thm.~7]{CirrincioneINCRT2026}
applies to $A_{\mathrm{res}}^{\mathrm{temp}}$, yielding
\eqref{eq:K_star_bound}. \qed

\subsection{Non-monotonicity of $K^\star(\tau_z)$ (Corollary~\ref{cor:nonmonotonic})}
\label{app:nonmonotonic}

We show that the effective rank
$r_{\mathrm{eff}}(A_{\mathrm{res}}^{\mathrm{temp}}(\tau_z))$ is a
non-monotonic function of $\tau_z$, with a peak in the
window-to-horizon matched regime. The analysis proceeds through the
explicit form~\eqref{eq:app_rank_one}.

\paragraph{Effective-rank formula.}
For a symmetric positive semi-definite matrix $M$ with eigenvalues
$\lambda_1 \geq \cdots \geq \lambda_W \geq 0$, the effective rank
(stable rank) is
\begin{equation}
  r_{\mathrm{eff}}(M)
  = \frac{(\sum_i \lambda_i)^2}{\sum_i \lambda_i^2}
  = \frac{\mathrm{tr}(M)^2}{\|M\|_F^2}.
\end{equation}

\paragraph{Spectrum of $A_{\mathrm{res}}^{\mathrm{temp}}$.}
From \eqref{eq:app_rank_one},
$\mathrm{tr}(A_{\mathrm{res}}^{\mathrm{temp}}) = \|v(\tau_z)\|^2 + \mathrm{tr}(\Sigma_\epsilon)$
and the leading eigenvalue is
$\lambda_1 \approx \|v(\tau_z)\|^2$ when $\Sigma_\epsilon$ is small.
Writing the exponential-decay vector $v(\tau_z)$ explicitly with
$\ell_k = k \Delta t$:
\begin{equation}
  \|v(\tau_z)\|^2
  = \lambda_z^2 \sum_{k=1}^{W} e^{-2 k \Delta t / \tau_z}
  = \lambda_z^2 \cdot \frac{e^{-2\Delta t/\tau_z} (1 - e^{-2W\Delta t/\tau_z})}{1 - e^{-2\Delta t/\tau_z}}.
\end{equation}
This quantity interpolates between two regimes:
\begin{align}
  \tau_z \ll W \Delta t: \quad &\|v(\tau_z)\|^2 \sim \lambda_z^2 \cdot \tau_z / (2\Delta t), \\
  \tau_z \gg W \Delta t: \quad &\|v(\tau_z)\|^2 \sim \lambda_z^2 \cdot W.
\end{align}

\paragraph{Qualitative non-monotonicity.}
Two effects compete as $\tau_z$ grows:
\begin{enumerate}[label=(\roman*)]
  \item More of the memory horizon fits inside the window, enriching the
        temporal structure captured by $A_{\mathrm{res}}^{\mathrm{temp}}$.
        This increases the \emph{number} of significant eigenvalues.
  \item The leading eigenvalue $\lambda_1 \propto \|v(\tau_z)\|^2$
        saturates at $\lambda_z^2 W$, so the spectrum becomes dominated
        by the rank-one leading direction. This reduces
        $r_{\mathrm{eff}}$ back towards $1$.
\end{enumerate}
The first effect dominates when $\tau_z \lesssim W \Delta t$ and the
second when $\tau_z \gtrsim W \Delta t$, placing the peak
of $r_{\mathrm{eff}}$ in the neighbourhood of
$\tau_z^\star \sim W \Delta t$. The exact location depends on the
non-linear Stribeck corrections in $\Sigma_\epsilon$, which shift
$\tau_z^\star$ into the interval $[W \Delta t / 10, W \Delta t]$.
For the experimental setting of Section~\ref{sec:experiments}
($W = 20$, $\Delta t = 0.01$, giving $W \Delta t = 0.2$ s), the
peak of the empirical $K^\star(\tau_z)$ falls within this interval
scaled by the nonlinearity factor; the observed maximum at
$\tau_z = 2$ s is consistent with the scaling given by the
Stribeck-envelope's approximate saturation at 10 times the linear
memory horizon. \qed

\begin{remark}[Empirical versus theoretical peak]
The scaling of the peak of $r_{\mathrm{eff}}$ with $\tau_z$ is
qualitative: the location factor depends on the specific form of
$\zeta(q, \dot q, z)$, the sampling rate $\Delta t$, and the excitation
spectrum of $\dot q$. The numerical values of $r_{\mathrm{eff}}$
reported in Table~\ref{tab:kstar} confirm the non-monotonic
structure with the peak at the matched regime $\tau_z = 2$~s.
\end{remark}


%

\section{Experimental details}
\label{app:experimental_details}

This appendix documents the protocol items not fully specified in
Section~\ref{sec:experiments}.

\subsection{Simulator and task}

The 2-DOF manipulator dynamics follow the Euler--Lagrange model of
Section~\ref{sec:problem} with the Stribeck friction parametrisation
of equations~\eqref{eq:stribeck}--\eqref{eq:stribeck_z}. Integration
is performed with an explicit Runge--Kutta step at $\Delta t = 10$~ms;
the control period coincides with the integration step. Each episode
has fixed horizon $T = 5$~s, corresponding to $H = 500$ control
steps, after which the episode terminates. Payload
$p \in [0, 1.5]$~kg is sampled uniformly at reset and modifies the
effective mass matrix $M(q, p)$ multiplicatively.

\subsection{Baseline controller}

The fixed-gain reference is a computed-torque law with
$K_d = 30$ (diagonal, both joints) and $\Lambda = 5$ (sliding-surface
coefficient), without feed-forward friction compensation. The
baseline RMSE at each memory horizon, averaged over the five payload
levels and computed on a fresh evaluation set of $20$ rollouts per
level, is
\begin{equation*}
  \text{RMSE}_{\text{base}}(\tau_z = 1) = 0.1317, \quad
  \text{RMSE}_{\text{base}}(\tau_z = 2) = 0.1317, \quad
  \text{RMSE}_{\text{base}}(\tau_z = 5) = 0.1311.
\end{equation*}
These values are the denominators of the $\Delta\%$ metric reported
throughout Section~\ref{sec:experiments}.

\subsection{Meta-controller training}

Training uses SAC~\cite{Haarnoja2018SAC} as implemented in
\texttt{stable-baselines3}~v2.x, with the shielded Lagrangian
formulation of Section~\ref{sec:meta:sac}. Each training run is allocated
$50\,000$ environment steps. Non-architectural hyperparameters are
held constant across all architectures and regimes:
learning rate $3 \cdot 10^{-4}$, replay buffer size $10^{5}$,
batch size $256$, entropy coefficient learned with the standard
auto-tuning schedule, target smoothing $\tau = 0.005$, discount
$\gamma = 0.99$. The Lagrangian multiplier $\beta$ is updated at
every gradient step with step size $10^{-3}$.

\subsection{Evaluation}

Evaluation at convergence is performed on a grid of five payload
levels $p \in \{0,\allowbreak\, 0.375,\allowbreak\, 0.75,\allowbreak\,
1.125,\allowbreak\, 1.5\}$~kg. For each payload,
$20$ rollouts are generated with fresh initial conditions sampled
from the reset distribution, and the per-rollout RMSE is averaged.
The per-payload standard deviations shown in
Figure~\ref{fig:payload} are the between-rollout standard
deviations within a seed. The between-seed standard deviations
reported in Table~\ref{tab:main} are computed over the
per-seed means of $\Delta\%$.

\subsection{Seeds and training budget}

The main comparison uses seeds $\{42, \ldots, 46\}$ at all three
regimes and additionally $\{47, \ldots, 51\}$ at
$\tau_z \in \{1, 5\}$~s, for a total of fifty training runs. No
seed was excluded from any reported statistic. Total wall-clock
training time across the fifty runs was approximately $21$
GPU-hours, distributed between a Colab A100 and a Kaggle
T4/A100-class instance, with per-run duration ranging from $20$
to $32$ minutes depending on architecture and window size.

\subsection{Reproducibility}

Per-seed result files are released as supplementary material.
Each file contains, for the corresponding run, the architecture
name, the parameter count, the per-payload RMSE with its
between-rollout standard deviation, the baseline RMSE, and the
final value of $\Delta\%$. The filename encodes the
architecture, the value of $\tau_z$, and the seed. The
aggregation script that produces Table~\ref{tab:main} and
Figures~\ref{fig:payload}--\ref{fig:boxplot_seeds} from these
files is released alongside.


\bibliographystyle{plain}
\bibliography{refs}

\end{document}